\documentclass{article}

\PassOptionsToPackage{numbers}{natbib}
\usepackage{natbib}
\usepackage[preprint]{neurips_2026}
\usepackage[utf8]{inputenc} 
\usepackage[T1]{fontenc}    
\usepackage{hyperref}       
\usepackage{url}            
\usepackage{booktabs}       
\usepackage{amsfonts}       
\usepackage{nicefrac}       
\usepackage{microtype}      
\usepackage{amsmath}
\usepackage{amsthm}
\usepackage{enumitem}
\usepackage{mathtools}
\usepackage{float}
\usepackage{wrapfig}
\usepackage{tabularx}
\usepackage{array}
\usepackage{makecell}
\usepackage{graphicx}
\usepackage{subcaption}
\usepackage{multirow}
\usepackage{algorithm}
\usepackage{algpseudocode}
\usepackage{adjustbox}
\usepackage[dvipsnames]{xcolor}
\usepackage{titletoc} 
\newtheorem{problem}{Problem}

\theoremstyle{plain}
\newtheorem{theorem}{Theorem}[section]

\theoremstyle{definition}

\theoremstyle{remark}

\newcommand{\ours}{TriGlue}

\title{TriGlue: a Biology-Inspired Generative Model for Generating Molecular Glue-Induced Ternary Complex}

\author{
 \textbf{Yuliang Yan\textsuperscript{1}},
 \textbf{Shuo Yan\textsuperscript{1}},
 \textbf{Haochun Tang\textsuperscript{2,1}},
 \textbf{Yiqin Sun}\textsuperscript{3},
 \textbf{Enyan Dai\textsuperscript{1}}
\\
 \textsuperscript{1}The Hong Kong University of Science and Technology (Guangzhou)\\
 \textsuperscript{2}Jilin University\\
 \textsuperscript{3}University of Sydney
\\
}

\begin{document}

\maketitle

\begin{abstract}
Molecular glue degraders have emerged as a promising strategy for targeted protein degradation by inducing ternary complex formation between an E3 ubiquitin ligase and a target protein. 
Despite their therapeutic potential, computational design of molecular glues remains largely unexplored. 
Unlike conventional structure-based drug design, molecular glue design is governed by the unknown protein-protein interface and requires the simultaneous modeling of ligand generation, protein-protein docking, and ternary complex assembly.
In this work, we formulate molecular glue design as a ternary complex generation problem and propose a biology-inspired generative framework, \textbf{\ours}. 
Motivated by the mechanism of molecular glue action, we decompose ternary complex generation into two coupled stages: interface estimation and interface-conditioned complex generation. 
First, we develop an SE(3)-equivariant interface estimation module that predicts a geometrically constrained protein-protein interface from unbound monomer structures. 
Second, we introduce an interface-conditioned ternary flow matching network that jointly generates the molecular glue and predicts the rigid-body transformation required to assemble the ternary complex. 
Extensive experiments demonstrate that \textbf{\ours} generates chemically valid molecules and produces plausible ternary complexes, which highlight the potential of biology-inspired generative modeling for accelerating molecular glue discovery. 
Our code is available at \url{https://github.com/yuliangyan0807/molecular-glue-design}.
\end{abstract}

\section{Introduction}\label{sec:intro}
 \begin{wrapfigure}{r}{0.54\linewidth}
    \vspace{-2.0em}
    \centering
    \includegraphics[width=\linewidth]{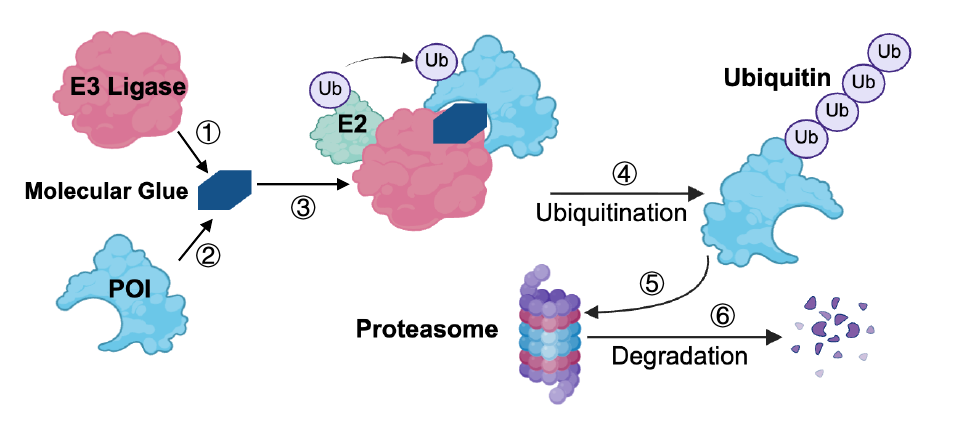}
    \vspace{-2.0em}
    \caption{MGDs in targeted protein degradation.
    }
    \label{fig:mgd&protac}
    \vspace{-1.0em}
\end{wrapfigure}

Recently, molecular glue degraders (MGDs) have emerged as a rapidly growing strategy in targeted protein degradation~\citep{mg1,mg4,mg5,mg2,mg3}. 
As illustrated in Fig.~\ref{fig:mgd&protac}, the mechanism of molecular glue degraders (MGDs) proceeds as follows: 
The molecular glue binds to the E3 ubiquitin ligase and remodels its interaction surface. 
This remodeled interface recruits and stabilizes the association of the target protein (POI), forming a ternary complex. 
Then, the E2 enzyme transfers ubiquitin molecules to the POI, resulting in ubiquitination, and the ubiquitinated POI is recognized by the proteasome and degraded, completing the targeted protein degradation process.
Molecular glues are typically structurally simple, drug-like, and increasingly clinically validated~\citep{mg9}. 
Therefore, molecular glue development has become a rapidly emerging focus in modern drug discovery~\citep{mg6,mg7,mg8}. 

Accelerating the discovery of molecular glues remains a central objective. 
Traditional molecular glue discovery has largely relied on expensive and labor-intensive experimental screening, and to date, only a few hundred molecular glues have been identified~\citep{mg10,mg12,mg11}. 
AI-based approaches have recently propelled drug discovery~\citep{targetdiff,decompdiff,uniguide,pocketxmol}. 
However, these methods are not suitable for molecular glue design, where ligand generation, protein docking, and ternary complex formulation must be considered simultaneously.
Fundamentally, they are designed for binary protein–ligand binding and typically assume the existence of a predefined binding pocket on a single protein. 
DeepTernary~\citep{deepternary} introduces an SE(3)-equivariant network for ternary complex structure prediction. 
However, it focuses on predicting structures for known ligands and does not address \textit{de novo} drug design. 
Therefore, we pursue generative drug design tailored to ternary molecular glue systems.

Molecular glue design is to generate a small molecule that induces an E3 ubiquitin ligase and a target protein to form a productive ternary complex. 
Achieving this goal requires the simultaneous optimization of three tightly coupled components: the molecule-induced E3 ligase--target protein interface, the novel small molecule, and the rigid-body transformation that positions the two proteins into a coherent ternary complex. 
This gives rise to two major challenges. 
\emph{First}, the success of molecular glue design is determined by the formation of a compatible ligand-mediated protein–protein interface, which ultimately dictates whether a stable ternary complex can form.
However, this interface is unknown prior to complex assembly and cannot be directly observed from the unbound proteins. 
\emph{Second}, molecular glue design requires interface-conditioned generative modeling that jointly produces a novel ligand and the rigid-body transformation aligning the two proteins simultaneously. 
This involves tightly coupled dependencies across three components: the ligand structure, the protein–protein docking pose, and the interface geometry that stabilizes their interaction. 

To address these challenges, we propose a biology-inspired generative framework named \textbf{\ours}.
(i) We develop an SE(3)-equivariant interface estimation module that predicts a protein–protein interface from the two monomer structures, providing a geometrically constrained and biologically meaningful structural prior.
(ii) We introduce an interface-conditioned ternary flow matching network that jointly generates the molecular glue and predicts the rigid-body rotation and translation required for protein–protein docking under the guidance of the estimated interface.
In summary, our main contributions are as follows:
\begin{itemize}[leftmargin=*]
    \item We investigate the novel problem of molecular glue design for ternary complexes.
    \item We propose a biology-inspired generative framework that decomposes ternary complex generation into interface estimation and interface-conditioned ternary flow generation.
    \item Extensive experiments demonstrate the effectiveness of the proposed framework in generating valid small molecules and plausible ternary complexes.
\end{itemize}
\section{Preliminaries}
\subsection{Molecular Glue-Induced Ternary Complex Formation}
\noindent \textbf{Mechanistic Overview.} Molecular glue drives ternary complex formation by binding to one protein and reshaping its surface to recruit a second partner, thereby inducing or stabilizing a protein–protein interaction. 
They remodel the interaction landscape and convert transient encounters into stable assemblies, rather than acting as simple binders. 
Molecular glues can modulate protein function, localization, or stability through cooperative interactions among the ligand and both proteins, which explains their growing importance in therapeutic discovery. 

\noindent\textbf{Molecular Glue-induced Protein-protein Interface.} Molecular glue activity is determined by the creation of a ligand-mediated protein–protein interface. 
The ligand occupies the interfacial region and stabilizes complementary surface patches from both proteins, effectively bridging them and reducing the tendency of the complex to dissociate. 
Therefore, the molecular glue-induced protein–protein interface serves as an important biological prior to form the ternary complex.

\noindent\textbf{Dynamic Ternary Complex Assembly.} 
Ternary complex formation is inherently dynamic. 
The two proteins typically start from unbound conformations and undergo rigid-body rearrangements as ligand binding stabilizes the complex. 
The final ternary structure therefore emerges from a coupled process involving ligand placement, protein docking, and cooperative stabilization. 
This dynamic view naturally motivates a generative formulation that jointly models ligand structure and protein motion, aligning the learning objective with the underlying biological mechanism.
\subsection{Problem Formulation}
Formally, we formulate molecular-glue-induced ternary complex design as a generative modeling task.
Specifically, a ternary complex is defined as $\mathcal{C} = (\mathcal{R}, \mathcal{M}, \mathcal{T})$, where $\mathcal{R}$ and $\mathcal{T}$ denote the receptor and target proteins, and $\mathcal{M}$ is a small molecule that induces their association.
Each protein is represented as a residue-level graph $\mathcal{G}=(\mathcal{V},\mathcal{E})$, where node $i$ corresponds to residue identity $\mathrm{AA}_i$ with C$\alpha$ coordinate $\mathbf{X}_i\in\mathbb{R}^3$, and edges are constructed using a $k$-nearest-neighbor scheme.
The ligand is represented as a molecular graph $\mathcal{G}_{\mathcal{M}}=(\mathcal{V}_{\mathcal{M}},\mathcal{E}_{\mathcal{M}})$, where each node corresponds to an atom type $\mathbf{a}_i$ with coordinate $\mathbf{x}_i\in\mathbb{R}^3$.
During generation, the bound ternary complex is unknown. We are given a receptor protein $\mathcal{R}$ in its bound conformation, an unbound target protein $\tilde{\mathcal{T}}$ with unknown pose, and a ligand $\mathcal{M}$ to be generated. 
The bound target structure is related to the unbound structure by a closed-form rigid-body transformation:
\begin{equation}\label{eq:rigid}
\mathbf{X}_{\mathcal{T}} = \mathbf{R}\mathbf{X}_{\tilde{\mathcal{T}}} + \mathbf{t},
\qquad
\mathbf{R} \in \mathrm{SO}(3),\;\mathbf{t} \in \mathbb{R}^3.
\end{equation}
\begin{problem}\label{pro}
Given a bound receptor protein $\mathcal{R}$ and an unbound target protein $\tilde{\mathcal{T}}$, molecular glue-induced ternary complex design is formulated as learning the conditional distribution:
\begin{equation}\label{eq:objective}
p(\mathcal{C} \mid \mathcal{R}, \tilde{\mathcal{T}}).
\end{equation}
\end{problem}

\begin{figure}[H]
    \centering
    \includegraphics[width=0.92\linewidth]{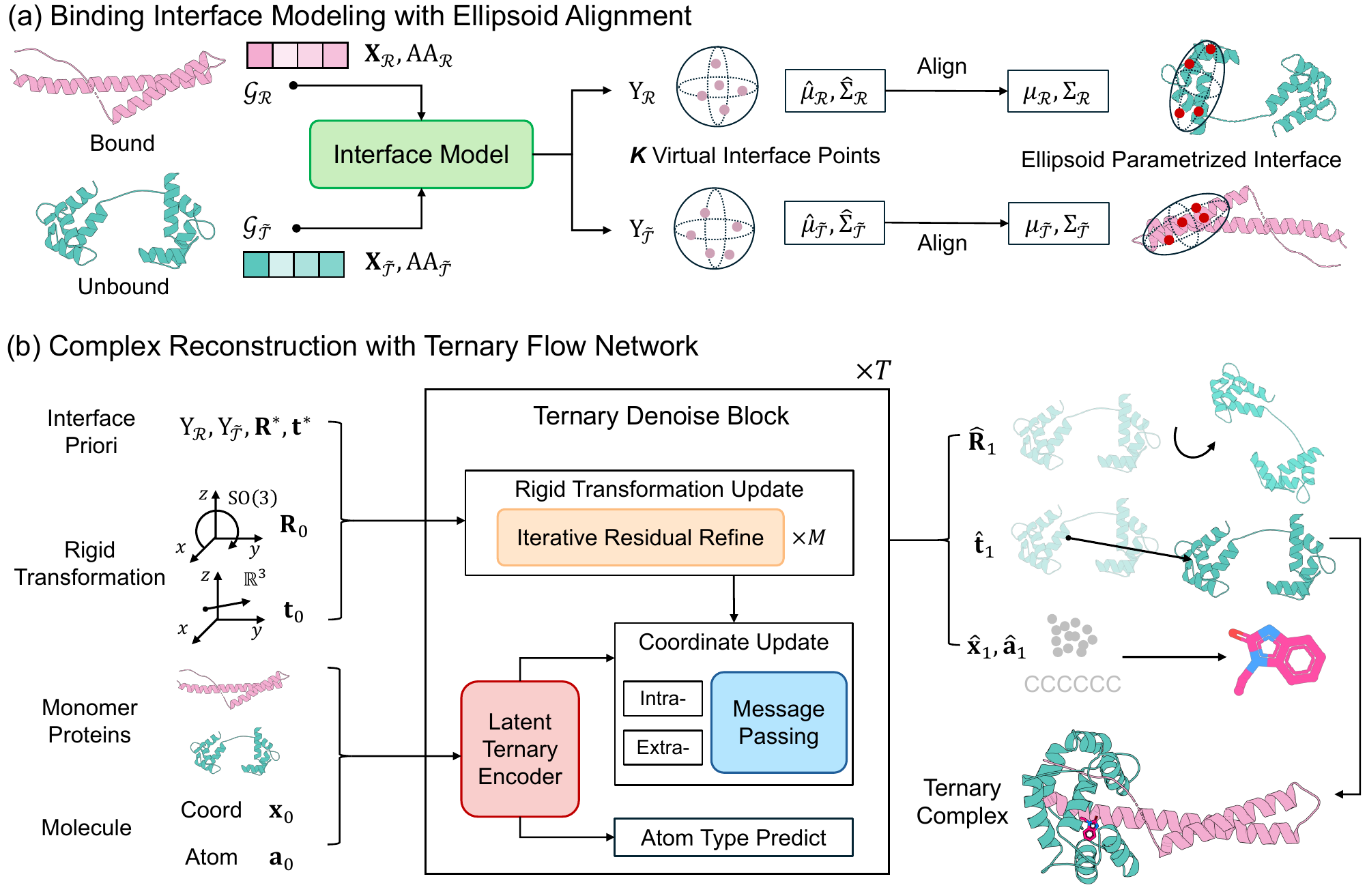}
    \caption{Framework of our interface-conditioned ternary flow network TriGlue.
    \vspace{-1.0em}
}
    \label{fig:overview}
\end{figure}

\section{Method}
Following the biological process of molecular glue-induced ternary complex formation, we formulate {\ours} as a biology-inspired generative framework. 
As shown in Fig.~\ref{fig:overview}, we factorize ternary complex generation into two coupled modules: an interface estimation module that infers a global virtual interface from two monomer proteins, and an interface-conditioned complex generation module that jointly generates the molecular glue and estimates the inter-protein rigid-body transformation to assemble a coherent ternary complex. 
This factorization aligns the generative process with the biological mechanism where a glue-mediated protein--protein interface guides subsequent ternary complex assembly.
The factorized design raises two key challenges. 
(i) the inferred interface should be constrained in a geometrically meaningful manner to ensure its utility in ternary complex modeling; 
(ii) the predicted interface should effectively guide a unified generative process that simultaneously produces a novel molecular glue and drives the rigid-body assembly of the two proteins. 
To address these challenges, we constrain the estimated interface with an ellipsoid-parameterized geometry. 
We further inject interface information through an interface-conditioned ternary flow network. 
We next describe the factorization process and the two corresponding modules in detail. 

\subsection{Biology-Inspired Factorization of Ternary Complex Generation}
Inspired by the mechanism of molecular glue-induced ternary complex formation, we decompose the Problem~\ref{pro} into two sub-goals: interface estimation and interface-conditioned complex generation. 
Specifically, this decomposition leads to two coupled probabilistic components: one estimates the molecular glue-induced interface between the two proteins, and the other generates the ternary complex conditioned on the inferred interface as follows:
\begin{equation}\label{eq:objective_factorized}
p(\mathcal{C} \mid \mathcal{R}, \tilde{\mathcal{T}})
=
\underbrace{p(\mathcal{I}_{\mathcal{R}}, \mathcal{I}_{\tilde{\mathcal{T}}}
\mid \mathcal{R}, \tilde{\mathcal{T}})}_{\text{Interface Estimation}}
 \cdot 
 \underbrace{p\big((\mathcal{R},\mathcal{M},\mathcal{T})
\mid \mathcal{I}_{\mathcal{R}}, \mathcal{I}_{\tilde{\mathcal{T}}},
\mathcal{R}, \tilde{\mathcal{T}}\big)}_{\text{Interface-Conditioned Complex Generation}}
,
\end{equation}
where $\mathcal{I}_{\mathcal{R}}$ and $\mathcal{I}_{\tilde{\mathcal{T}}}$ denote the sets of interface residues, whose C$\alpha$ atoms lie within $8\mathrm{\mathring{A}}$ of the binding partner in the complex. 
This biology-inspired factorization brings two key benefits: (i) It faithfully matches the mechanism of molecular-glue action; (ii) It transforms the challenging joint modeling problem into two tractable sub-tasks. 
We next present our approach for modeling both components.

\subsection{Interface Estimation under Ellipsoidal Geometric Constraints}\label{sec:interface_predict}
Protein–protein interface governs the process of molecular glue design, yet such information is unavailable when only monomer structures are given.  
This absence prevents the direct incorporation of biologically grounded geometric constraints during generation. 
We introduce an interface estimation module that infers a \emph{Global Virtual Interface} from the two monomer structures. 
We then align this virtual interface with an \emph{Ellipsoid Parameterized Representation} of the physical protein–protein interface. 
In the following, we present the detailed formulation of this module. 

\noindent \textbf{Protein Encoder}. 
Characterizing the interface first requires encoding each monomer protein to obtain latent representations. 
Given two proteins $\mathcal{G}_{\mathcal{R}}$ and $\mathcal{G}_{\tilde{\mathcal{T}}}$, we encode them with EGNN~\citep{egnn} to obtain the invariant features $\mathbf{F}_{\mathcal{R}} \in \mathbb{R}^{d \times N},\mathbf{F}_{\tilde{\mathcal{T}}} \in \mathbb{R}^{d \times M}$ and SE(3)-equivariant coordinates $\mathbf{Z}_{\mathcal{R}} \in \mathbb{R}^{3 \times N},\mathbf{Z}_{\tilde{\mathcal{T}}} \in \mathbb{R}^{3 \times M}$ as follows:
\begin{equation}   \mathbf{F}_{\mathcal{R}},\mathbf{Z}_{\mathcal{R}},\mathbf{F}_{\tilde{\mathcal{T}}},\mathbf{Z}_{\tilde{\mathcal{T}}}=\mathrm{EGNN}(\mathcal{G}_{\mathcal{R}},\mathcal{G}_{\tilde{\mathcal{T}}}),
\end{equation}
where $N, M$ is the node (residue) number and $d$ is the hidden dimension.

\noindent \textbf{Global Virtual Interface}.
To capture protein–protein interaction regions in a compact, geometry-aware manner, we map each protein to a fixed set of $K$ \emph{virtual interface points}~\citep{equidock}, denoted as $\mathbf{Y}_{\mathcal{R}}, \mathbf{Y}_{\tilde{\mathcal{T}}} \in \mathbb{R}^{3 \times K}$. 
The interface modeling in Eq.~(\ref{eq:objective_factorized}) is further approximated as: $p(\mathbf{Y}_{\mathcal{R}}, \mathbf{Y}_{\tilde{\mathcal{T}}}
\mid \mathcal{R}, \tilde{\mathcal{T}})$. 
Specifically, we represent each virtual point as an attention-weighted aggregation of coordinates:
\begin{equation}
\mathbf{y}_{\mathcal{R},k}
=
\sum_{i=1}^{N}
\alpha_{\mathcal{R},i}^{k}
\, \mathbf{z}_{\mathcal{R},i},
\quad
\mathbf{y}_{\tilde{\mathcal{T}},k}
=
\sum_{j=1}^{M}
\alpha_{\tilde{\mathcal{T}},j}^{k}
\, \mathbf{z}_{\tilde{\mathcal{T}},j},
\end{equation}
where the attention weights are computed via cross-protein feature attention:
\begin{equation}
\alpha_{\mathcal{R},i}^{k}
=
\mathrm{softmax}
\left(
\frac{1}{\sqrt{d}}\,
\mathbf{f}_{\mathcal{R},i}^{\top}
\mathbf{W}_{\mathcal{R}}^{k}
\, \phi(\mathbf{F}_{\tilde{\mathcal{T}}})
\right),
\quad
\alpha_{\tilde{\mathcal{T}},j}^{k}
=
\mathrm{softmax}
\left(
\frac{1}{\sqrt{d}}\,
\mathbf{f}_{\tilde{\mathcal{T}},j}^{\top}
\mathbf{W}_{\tilde{\mathcal{T}}}^{k}
\, \phi(\mathbf{F}_{\mathcal{R}})
\right).
\end{equation}
Here, $\mathbf{y}_{\mathcal{R},k}$ denotes the $k$-th virtual point of protein $\mathcal{R}$, 
$\mathbf{f}_{\mathcal{R},i}$ and $\mathbf{z}_{\mathcal{R},i}$ denote the feature and coordinate of residue $i$, 
$\mathbf{W}_{\mathcal{R}}^{k} \in \mathbb{R}^{d \times d}$ is a learnable projection matrix for head $k$, 
and $\phi(\cdot)$ is an average pooling operator. 
The same formulation is applied symmetrically to the protein $\tilde{\mathcal{T}}$. 
Our global virtual interface summarizes a soft surface region, 
yielding a unified interface representation. 

\noindent \textbf{Ellipsoid Parametrized Interface Alignment}. 
The underlying interaction region is expected to satisfy finer-grained geometric priors that capture its precise position and directional spread in 3D space. 
We introduce an ellipsoid-parametrized representation of the ground-truth interface and encourage the global virtual interface $\mathbf{Y}$ to reflect the geometric structure. 
Specifically, given the interface residue set $\mathcal{I}$, we summarize its spatial distribution using a Gaussian ellipsoid parameterized by a centroid $\boldsymbol{\mu}$ and a covariance matrix $\boldsymbol{\Sigma}$, which capture the location and anisotropic extent of the interaction region, respectively. 
More details are included in Appendix~\ref{app:ellip}.

To parameterize the geometry of the virtual interface, we compute the empirical mean and covariance of the predicted virtual points as follows:
\begin{equation}
\hat{\boldsymbol{\mu}}
=
\frac{1}{K}
\sum_{k=1}^{K} \mathbf{y}_k,
\qquad
\hat{\boldsymbol{\Sigma}}
=
\frac{1}{K}
\sum_{k=1}^{K}
(\mathbf{y}_k - \hat{\boldsymbol{\mu}})
(\mathbf{y}_k - \hat{\boldsymbol{\mu}})^{\top}.
\end{equation}
We align the virtual interface with the ground-truth ellipsoid by minimizing the discrepancy between their first- and second-order moments:
\begin{equation}
\mathcal{L}_{\mathrm{ellip}}
=
\|\hat{\boldsymbol{\mu}} - \boldsymbol{\mu}\|_2^2
+
\|\hat{\boldsymbol{\Sigma}} - \boldsymbol{\Sigma}\|_F^2 .
\end{equation}
The ellipsoid-parameterized interface provides conditioning information for subsequent generative design. 
We provide the proof of the equivariance in Appendix~\ref{app:proof_virtual_interface}.

\subsection{Interface-Conditioned Generation with Ternary Flow Network}
Ternary complex formation requires jointly determining the molecular glue structure and the rigid-body docking of the two proteins. 
Effectively integrating the learned interface prior into the generative process remains non-trivial. 
We therefore introduce an interface-conditioned generative framework that jointly models molecular conformation and protein–protein docking within a unified multi- modality flow matching formulation.
Formally, conditioned on the learned interface, the complex generation distribution is further factorized as three independent components:
\begin{align}
p\big((\mathcal{R},\mathcal{M},\mathcal{T})
\mid \underbrace{\mathcal{I}_{\mathcal{M}}, \mathcal{I}_{\tilde{\mathcal{T}}},
\mathcal{R}, \tilde{\mathcal{T}}}_{\text{condition} \coloneq \star}\big) 
&= 
p\big(\mathcal{M} \mid \mathcal{I}_{\mathcal{R}}, \mathcal{I}_{\tilde{\mathcal{T}}},
\mathcal{R}, \tilde{\mathcal{T}}\big)
\cdot
p\big(\mathcal{T} \mid \mathcal{I}_{\mathcal{R}}, \mathcal{I}_{\tilde{\mathcal{T}}},
\mathcal{R}, \tilde{\mathcal{T}}\big) \notag \\
&=
p(\mathbf{x} \mid \star)
\cdot
p(\mathbf{a} \mid \star)
\cdot
p(\mathbf{R},\mathbf{t} \mid \star).
\end{align}
Generally, three flows are modeled, each equipped with a specialized architecture: 
(i) \emph{Rigid Transformation Flow} models $p(\mathbf{R},\mathbf{t} \mid \star)$ and predicts the rigid-body rotation and translation that place the target protein into its ternary-complex pose.
(ii) \emph{Coordinate Flow} models $p(\mathbf{x} \mid \star)$ and generates the three-dimensional coordinates of ligand atoms.
(iii) \emph{Atom Type Flow} models $p(\mathbf{a} \mid \star)$ and predicts the chemical identities of ligand atoms. 
Details of these components are described next.

\noindent \textbf{Rigid Transformation Flow for $p(\mathbf{R},\mathbf{t} \mid \star)$}. 
Predicting the rigid-body transformation that restores the target protein pose is highly dependent on interface geometry, making direct estimation challenging.
We therefore use the global virtual interfaces to obtain an initial docking transformation and further refine it through an iterative module to approximate the final rigid-body configuration. 

Conversely, we first utilize our interface model to generate the virtual interface $\mathbf{Y}_{\mathcal{R}}$ and $\mathbf{Y}_{\tilde{\mathcal{T}}}$.
Then, we predict the \emph{first-order} rigid transformation with the \emph{Differentiable Kabsch Model}~\citep{equidock}. 
We form the cross-covariance matrix as $\mathbf{A}=\mathbf{Y}_{\mathcal{\tilde{\mathcal{T}}}} \mathbf{Y}_{\mathcal{R}}$, and decompose it with single value decomposition as $\mathbf{A}=\mathbf{U}_{2}\mathbf{S}\mathbf{U}_{1}^{\top}$. 
The resulting rotation $\mathbf{R}^{*}$ and translation $\mathbf{t}^{*}$ are given by:
\begin{equation}
    \mathbf{R}^{*} = \mathbf{U}_{2} \begin{pmatrix}
1 & 0 & 0 \\
0 & 1 & 0 \\
0 & 0 & d
\end{pmatrix} \mathbf{U}_{1}^{\top}, \qquad
\mathbf{t}^{*}= \hat{\mu}_{\tilde{\mathcal{T}}}-\mathbf{R}^{*} \hat{\mu}_{\mathcal{R}}, \qquad
d=\mathrm{sign}(\mathrm{det}(\mathbf{U}_{2}\mathbf{U}_{1}^{\top})).
\end{equation}
The first-order transformation $(\mathbf{R}^{*},\mathbf{t}^{*})$ and the virtual interface $\mathbf{Y}_{\tilde{\mathcal{T}}}$ are subsequently refined through an iterative Rigid Transformation (RT) block $\Phi_{\text{RT}}(\cdot)$. 
At each iteration, the $i$th layer of the module predicts \emph{residual updates} in the Lie algebra space and translation, conditioned on the current transformation, noised input, and the two virtual interfaces. The refinement proceeds as:
\begin{align}\label{eq:interface1}
     \Delta\omega,\Delta\mathbf{t}&=
    \Phi_{\text{RT}}^{(i)}(\mathbf{R}_{t},\mathbf{t}_{t},\textcolor{RedOrange}{\mathbf{R}^{*}},\textcolor{RedOrange}{\mathbf{t}^{*}},\textcolor{RedOrange}{\mathbf{Y}_{\mathcal{R}}},\textcolor{RedOrange}{\mathbf{Y}_{\mathcal{\tilde{\mathcal{T}}}}}),\\
     \mathbf{R}^{*} &= \exp([\Delta\omega\times])\mathbf{R}^{*},\\
     \mathbf{t}^{*} &= \mathbf{t}^{*} + \Delta\mathbf{t},\\
     \mathbf{Y}_{\mathcal{\tilde{\mathcal{T}}}} &= \mathbf{R}^{*}\mathbf{Y}_{\mathcal{\tilde{\mathcal{T}}}} + \mathbf{t}^{*},
\end{align}
where $\Delta\omega$ is a predicted tangent vector in the tangent space $\mathfrak{so}_3$, and $\exp(\cdot)$ denotes the exponential map from $\mathfrak{so}_3$ to $\mathrm{SO}(3)$. 
The noisy input $(\mathbf{R}_t,\mathbf{t}_t)$ is obtained from two conditional flows: 
(i) For rotations, given noise $\mathbf{R}_0 \sim U(\mathrm{SO}(3))$ and the target rotation $\mathbf{R}_1 \sim p(\mathbf{R}\mid\star)$, the Riemannian conditional flow~\citep{manifoldflowmatching} is defined as the geodesic interpolation between the $\mathbf{R}_0$ and $\mathbf{R}_1$:
\begin{equation}
\mathbf{R}_t=\psi_{t}(\mathbf{R}_{0} \mid \mathbf{R}_{1})=\exp_{\mathbf{R}_0}\big(t\cdot \log_{\mathbf{R}_0}(\mathbf{R}_1)\big),
\end{equation}
where $\exp$ and $\log$ denote the Riemannian exponential and logarithm maps on $\mathrm{SO}(3)$. 
$t$ is the timestep. 
(ii) For translations, given noise $\mathbf{t}_0 \sim \mathcal{N}(\mathbf{0},\sigma^2\mathbf{I})$ and ground-truth $\mathbf{t}_1 \sim p(\mathbf{t}\mid\star)$, the Euclidean conditional flow~\citep{euclideanflow} is defined as a linear interpolation as follows:
\begin{equation}
\mathbf{t}_t =\psi_{t}(\mathbf{t}_{0} \mid \mathbf{t}_{1})= t\cdot\mathbf{t}_1 + (1-t)\cdot\mathbf{t}_{0}.
\end{equation}
The conditional flow matching is utilized to train the model, and the objective is formulated as:
\begin{equation}
    \mathcal{L}_{\text{RT}}=\mathbb{E}_{t \sim \mathcal{U}(0,1)}
    \big[
    \| \left(\log_{\mathbf R_t}(\hat{\mathbf R}_1) - \log_{\mathbf R_t}(\mathbf{R}_1)\right) /
(1-t) \|_2^2
    +
    \|
    \hat{\mathbf{t}}_{1}-\mathbf{t}_{1}
    \|_2^2
    \big].
\end{equation}
where $(\hat{\mathbf{R}}_{1},\hat{\mathbf{t}}_{1})$ is the final output prediction of $\Phi_{\text{RT}}(\cdot)$. 

\noindent \textbf{Latent Ternary Encoder}. 
Modeling ternary complexes requires representations that remain invariant to input structures while placing all components in a unified latent space to capture their interactions.  
To be specific, given the input $\mathcal{G}_{\mathcal{R}}$, $\mathcal{G}_{\tilde{\mathcal{T}}}$, and $\mathcal{G}_{\tilde{\mathcal{M}}}$ (noised graph from $\mathcal{G}_{\mathcal{M}}$), we encode them with the \emph{Invariant Point Attention} block~\citep{alphafold}:
\begin{equation}
    \mathbf{H}_{\mathcal{R}}, \mathbf{H}_{\tilde{\mathcal{T}}},\mathbf{h}_{\tilde{\mathcal{M}}}=
    \mathrm{IPA}(\mathcal{G}_{\mathcal{R}},\mathcal{G}_{\tilde{\mathcal{T}}},\mathcal{G}_{\tilde{\mathcal{M}}}),
\end{equation}
where $\mathbf{H}_{\mathcal{R}}, \mathbf{H}_{\tilde{\mathcal{T}}},~ \text{and}~\mathbf{h}_{\tilde{\mathcal{M}}}$ denote invariant representations that encode both sequence and structural information, which are subsequently used by the denoising networks. 

\noindent \textbf{Coordinate Flow for $p(\mathbf{x} \mid \star)$}. 
Generating ligand atomic coordinates from random noise is particularly challenging when the binding pocket is unknown. 
We first leverage the predicted rigid transformation to pre-align the target protein, providing a geometry that is consistent with the inferred interface. 
Then, we integrate the unified latent representations with the current coordinates and design two complementary message passing: one capturing interactions between the proteins and the ligand, and the other modeling intra-ligand geometric dependencies, to progressively update coordinates. 

Specifically, the first-order transformation is utilized to correct the coordinate of the target protein: 
\begin{equation}\label{eq:interface2}
    \underbrace{\hat{\mathbf{X}}_{\mathcal{T}} = \textcolor{RedOrange}{\mathbf{R}^{*}}\mathbf{X}_{\tilde{\mathcal{T}}} + \textcolor{RedOrange}{\mathbf{t}^{*}}}_{\text{Interface-Correction}}.
\end{equation}
Then, a bi-level message passing network is constructed to denoise the molecular coordinates. 
Without loss of generality, for each $\mathbf{x}_{i} \in \mathbf{x}$, we adopt the linear interpolation in Euclidean space to obtain the noised coordinate as: $\mathbf{x}_{i,t}=t\mathbf{x}_{i,1}+(1-t)\mathbf{x}_{i,0}$, where $\mathbf{x}_{i,0} \sim \mathcal{N}(0,I_{3})$ is the prior and $\mathbf{x}_{i,1} \sim p(\mathbf{x}_{i})$ is the data distribution. 
Then, we update the coordinate via the intrinsic self-update and the extrinsic protein-molecule update with the Coordinate block 
$\Phi_{\mathbf{x}}(\cdot)$:
\begin{align}
\hat{\mathbf{x}}_{i,1} &= 
\mathbf{x}_{i,t}
+
\underbrace{
\sum_{j\in \mathcal{V}_{I}}
(\mathbf{x}_{i,t} - \mathbf{x}_{j,t})
\phi_{\text{intra}}
\big(
\mathbf{h}_{i}, \mathbf{h}_{j},
\|\mathbf{x}_{i,t} - \mathbf{x}_{j,t}\|_2^2,
t
\big)
}_{\text{Intrinsic Update}} \nonumber \\ 
&\quad +
\underbrace{
\sum_{k \in \mathcal{V}_{E}}
(\mathbf{x}_{i,t} - \mathbf{X}_k)\,
\phi_\text{extra}
\big(
\mathbf{h}_{i}, \mathbf{H}_{k},
\|\mathbf{x}_{i,t} - \mathbf{X}_k\|_2^2,
t
\big)
}_{\text{Extrinsic Update}},
\end{align}
where $\phi(\cdot)$ is MLP.  $\mathcal{V}_{I}$ and $\mathcal{V}_{E}$ denote the sets of neighbors for node $i$ corresponding to intra-molecular connections within the ligand and inter-molecular interactions with the two proteins, respectively. 
The $\Phi_{\mathbf{x}}$ block outputs the coordinate $\hat{\mathbf{x}}_1$, and the objective is formulated as follows:
\begin{equation}
    \mathcal{L}_{\text{coor}}=
    \mathbb{E}_{t \sim \mathcal{U}(0,1)}
    \big[
    \|
    \hat{\mathbf{x}}_{i,1}-\mathbf{x}_{i,1}
    \|_{2}^{2}
    \big].
\end{equation}

\noindent \textbf{Atom Type Flow for $p(\mathbf{a} \mid \star)$}. 
Recovering ligand atom identities is formulated as a continuous flow on the probability simplex. 
Each atom's identity $\mathbf{a}_i \in \mathbf{a}$ is mapped to a scaled soft one-hot logit vector~\citep{pepflow}, which lies on the probability simplex $\Delta^{|\mathcal{A}|-1}$, where $|\mathcal{A}|$ denotes the total number of atom types. 
The conditional flow is defined as a linear interpolation between the prior $\mathbf{a}_{i,0} \sim \mathcal{N}(\mathbf{0}, K^2 I_{\mid\mathcal{A}\mid})$ and the target $\mathbf{a}_{i,1}$:
$ \boldsymbol{\psi}_{t}(\mathbf{a}_{i,0} \mid \mathbf{a}_{i,1}) = t \mathbf{a}_{i,1} + (1 - t) \mathbf{a}_{i,0}$, and $K$ is a scale constant. 
The atom type is predicted through the Atom Type block $\Phi_{\mathbf{a}}(\cdot)$ and the training objective is:
{
\begin{equation}
\hat{\mathbf{a}}_1 = \Phi_{\mathbf{a}}
(
\mathbf{H}_{\mathcal{R}},
\mathbf{H}_{\tilde{\mathcal{T}}},
\mathbf{H}_{\tilde{\mathcal{M}}},
\mathbf{a}_t,
t
), \quad
    \mathcal{L}_{\text{type}}=
    \mathbb{E}_{t \sim \mathcal{U}(0,1)}
    \big[
    \|
    \hat{\mathbf{a}}_{1}-\mathbf{a}_{1}
    \|_{2}^{2}
    \big].
\end{equation}
}

\noindent \textbf{Training Objective}.
Finally, we balance the gradients of all flows using weighting coefficients $\lambda_{i}$ and jointly train the entire ternary denoising framework as follows: 
\begin{equation}
    \mathcal{L}_{\text{flow}}=(\lambda_1,\lambda_2)\cdot \mathcal{L}_{\text{RT}} +\lambda_3\cdot\mathcal{L}_{\text{coor}}+\lambda_4 \cdot \mathcal{L}_{\text{type}}.
\end{equation}
We provide the details of the model architecture and the flow matching in the Appendix~\ref{app:architecture} and \ref{app:flow_matching}.

\noindent \textbf{Sampling}. We provide the detailed sampling algorithm in the Appendix~\ref{app:sample}.
\section{Experiment}
In this section, we conduct experiments to answer the following research questions:
\begin{itemize}[leftmargin=*]
    \item \textbf{RQ1: }Can {\ours} generate high-quality molecular glues while accurately docking the two proteins into a reasonable ternary complex?
    \item \textbf{RQ2: }How accurate and geometrically meaningful are the ellipsoid interfaces predicted by {\ours}?
    \item  \textbf{RQ3: }How does interface conditioning affect the performance of {\ours}?
\end{itemize}

\subsection{Experimental Setup}\label{sec:experiment_setting}
\noindent \textbf{Dataset}. 
We utilize \textbf{TernaryDB}~\citep{deepternary} for training and evaluation, which contains $22,303$ ternary complexes curated from the Protein Data Bank~\citep{pdb}, covering diverse proteins and drug-like ligands. 
Moreover, to prevent data leakage, complexes are clustered by protein sequence similarity using MMseqs2~\citep{mmseqs2}, and clusters overlapping with known molecular glue complexes are excluded from training to construct non-overlapping validation and test splits. 
Detailed statistics are in Appendix~\ref{app:dataset_details}

\noindent \textbf{Baselines}. 
Since no existing method jointly performs \emph{de novo} molecular glue generation and ternary complex reconstruction, we evaluate representative baselines separately for each sub-task. 
For molecular generation, we compare with three structure-based design methods, \textbf{AR}~\citep{ar}, \textbf{TargetDiff}~\citep{targetdiff}, and \textbf{PocketXMol}~\citep{pocketxmol}. 
For protein-protein docking, we adopt \textbf{DeepTernary}~\citep{deepternary}, \textbf{EquiDock}~\citep{equidock}, and \textbf{DiffDock-PP}~\citep{diffdockpp}. 
Detailed baseline descriptions are provided in Appendix~\ref{app:baseline}.

\noindent \textbf{Training Details}. 
For the interface modeling module, the EGNN encoder consists of 4 layers with a hidden dimension of 128. 
The model is trained using the AdamW optimizer~\citep{adamw} with an initial learning rate of $5\times10^{-4}$ and a cosine annealing schedule for 1200 epochs. 
For the ternary denoising network, we stack two IPA blocks as the encoder with a hidden dimension of 128. 
Training is performed using the AdamW optimizer with a linear learning rate decay from $5\times10^{-4}$ to $4\times10^{-4}$ for 800 epochs. 
The loss weights for the joint training objective are set to $\{0.008, 1.0, 1.0, 10.0\}$.
All experiments are conducted on 8 NVIDIA L40 GPUs. 
Appendix~\ref{app:train_details} provides more details.

\noindent \textbf{Evaluation Metrics}. 
We evaluate performance from two perspectives. 
For molecular generation, we report AutoDock Vina-based affinity metrics (\textbf{Vina}, \textbf{Vina min}, \textbf{Vina dock})~\citep{autodock}, the high-affinity ratio (\textbf{Aff}), and molecular property metrics (\textbf{QED}, \textbf{SA}). 
For ternary-complex reconstruction, we measure ligand pose accuracy using \textbf{RMSD} and protein docking quality using \textbf{DockQ}~\citep{dockq} and its success rate (DockQ $>0.23$). 
We provide more details in Appendix~\ref{app:metric}. 

\begin{table}[H]
  \centering
  \small
  \setlength{\tabcolsep}{3.5pt}
  \renewcommand{\arraystretch}{1.05}
  \caption{Performance comparison on generated molecules.}
  \label{tab:main_table}

  \begin{tabular}{lcccccccccccc}
    \toprule
    \multirow{2}{*}{\textbf{Method}}
      & \textbf{RMSD}~($\downarrow$)
      & \multicolumn{2}{c}{\textbf{Vina score~($\downarrow$})}
      & \multicolumn{2}{c}{\textbf{Vina min~($\downarrow$})}
      & \multicolumn{2}{c}{\textbf{Vina dock~($\downarrow$})}
      & \multicolumn{2}{c}{\textbf{QED~($\uparrow$})}
      & \multicolumn{2}{c}{\textbf{SA~($\uparrow$)}}
      & \textbf{Aff~($\uparrow$)} \\

    \cmidrule(lr){2-2}
    \cmidrule(lr){3-4}
    \cmidrule(lr){5-6}
    \cmidrule(lr){7-8}
    \cmidrule(lr){9-10}
    \cmidrule(lr){11-12}
    \cmidrule(lr){13-13}

      & Avg
      & Avg & Med
      & Avg & Med
      & Avg & Med
      & Avg & Med
      & Avg & Med
      & Avg \\

    \midrule
    AR & 8.16$_{2.17}$ & 8.76 & 9.04 & 3.62 & 3.69 & -4.39 & -4.87 & 0.51 & 0.54 & 0.54 & 0.53 & 0.53 \\
    TargetDiff & 6.27$_{1.99}$ & -2.11 & -6.09 & -4.83 & -6.74 & -7.55 & -7.86 & 0.47 & 0.53 & 0.55 & 0.58 & 0.65 \\
    PocketXMol & 6.74 $_{0.91}$  & -6.17 & -6.14 & \textcolor{RubineRed}{-6.88} & -6.88 & -7.69 & -7.65 & 0.47 & 0.49 & \textcolor{RubineRed}{0.79} & \textcolor{RubineRed}{0.78} & 0.69 \\
    {\ours} & \textcolor{RubineRed}{5.22 $_{1.41}$ } & \textcolor{RubineRed}{-6.23} & \textcolor{RubineRed}{-6.79} & {-6.77} & \textcolor{RubineRed}{-7.37} & \textcolor{RubineRed}{-8.28} & \textcolor{RubineRed}{-8.85} & \textcolor{RubineRed}{0.67} & \textcolor{RubineRed}{0.77} & 0.58 & 0.66 & \textcolor{RubineRed}{0.79} \\
    \bottomrule
  \end{tabular}
\end{table}

\subsection{Ternary Complex Generation}
To answer \textbf{RQ1}, the evaluation is conducted from two complementary perspectives: 
\begin{itemize}[leftmargin=*]
    \item \textbf{The quality and placement of the generated molecules.} Here, all baselines are provided with pocket information and use their best sampling strategies. 
    As shown in Tab.~\ref{tab:main_table}, even without access to binding pockets, {\ours} achieves the most accurate ligand placement in the reconstructed complexes, yielding the lowest average RMSD (5.22).
    Meanwhile, {\ours} generates ligands with strong predicted binding affinity: all affinity-related metrics outperform the baselines except for Vina min, which is slightly lower than PocketXMol. 
    In terms of molecular properties, {\ours} produces more drug-like molecules (QED 0.67), while the SA score is worse than PocketXMol. 
    We attribute this to PocketXMol being trained on a substantially larger dataset ($\sim10^{8}$ molecules), which provides broader coverage of chemical space.
    \item \textbf{The docking accuracy of the unbound target protein.} DeepTernary predicts rigid transformations from complex sequences, whereas DiffDock and EquiDock take the two protein structures as input. 
    As shown in Fig.~\ref{fig:dockq},
    compared with DeepTernary, {\ours} achieves better overall performance with a higher acceptance rate (35.2\% vs. 30.8\%). Although DeepTernary attains a higher best-case performance, its results exhibit a heavier long-tail distribution. 
\end{itemize}

\begin{figure}[H]
    \centering    \includegraphics[width=0.90\linewidth]{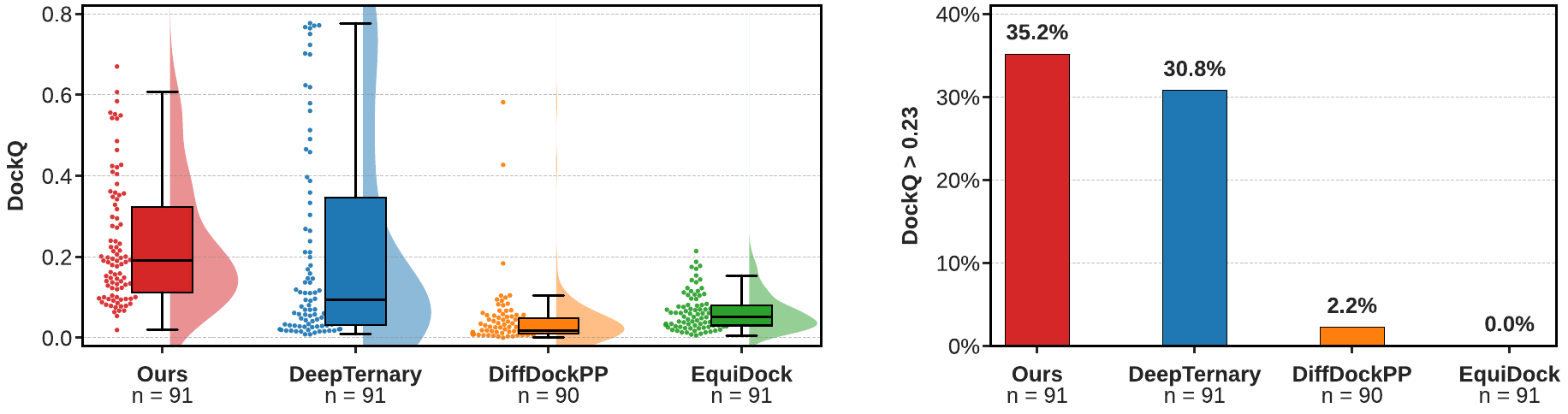}
    \caption{DockQ results.}
    \label{fig:dockq}
\end{figure}

\begin{figure}[H]
\vspace{-2.0em}
    \centering
    \includegraphics[width=1.0\linewidth]{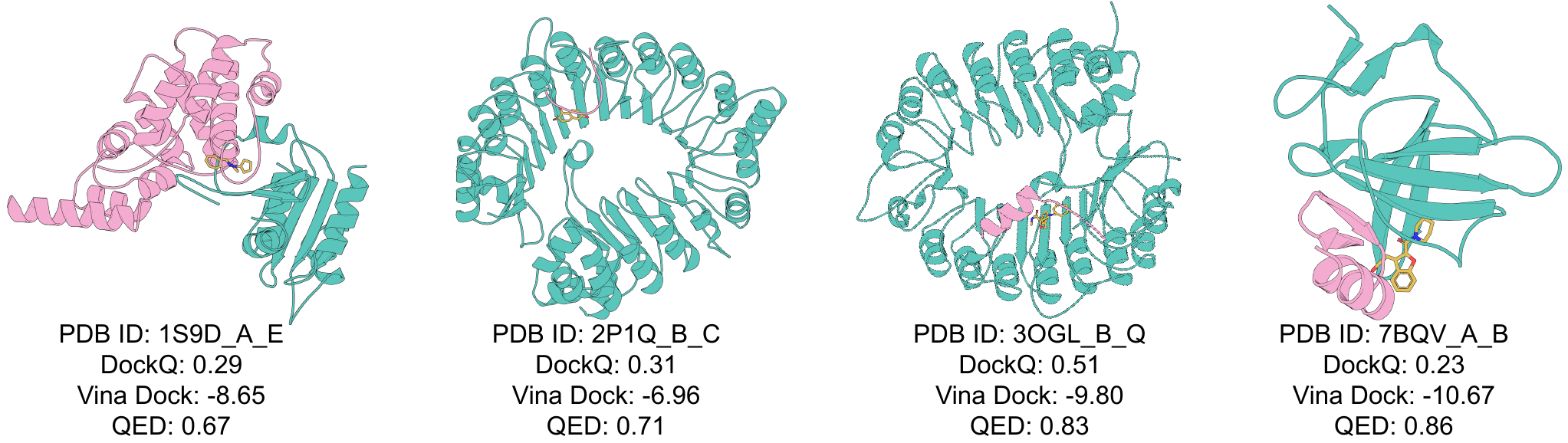}
    \caption{Serval ternary complexes generated by {\ours}.}
    \label{fig:examples}
\end{figure}

\subsection{Ellipsoid-Parameterized Interface Prediction} 
To further answer \textbf{RQ2}, we evaluate estimated interface ellipsoids on the held-out TernaryDB test split and compare against \textbf{PeSTo}~\citep{pesto} and \textbf{ScanNet}~\citep{scannet}. 
Residue-level interface predictions from these methods are converted into ellipsoids. 
We report geometry- and distribution-aware metrics, including JSD, Wasserstein-2 distance (W2), Center-L2, and LogDet error (see Appendix~\ref{app:ellipsoid_metric}). 
As shown in Tab.~\ref{tab:ellipsoid_interface_metrics}, our method consistently outperforms both baselines across all metrics, achieving substantially lower divergence and improved geometric alignment with ground-truth interface ellipsoids.  

\begin{table}[H]
\centering
\small
\setlength{\tabcolsep}{5pt}
\renewcommand{\arraystretch}{1.07}
\caption{Ellipsoid interface estimation on the test set.}
\label{tab:ellipsoid_interface_metrics}
\begin{tabular}{lcccc}
\toprule
\textbf{Method} & \textbf{JSD $\downarrow$} & \textbf{W2 Dist $\downarrow$} & \textbf{Center-L2 $\downarrow$} & \textbf{LogDet Error $\downarrow$} \\
\midrule
Ours & \textcolor{RubineRed}{0.47 $\pm$ 0.15} & \textcolor{RubineRed}{12.65 $\pm$ 7.89} & \textcolor{RubineRed}{10.99 $\pm$ 7.47} & \textcolor{RubineRed}{3.83 $\pm$ 5.01} \\
PeSTo   & 0.68 $\pm$ 0.04 & 58.26 $\pm$ 46.03 & 57.39 $\pm$ 46.25 & 8.00 $\pm$ 11.54 \\
ScanNet & 0.68 $\pm$ 0.04 & 58.96 $\pm$ 45.34 & 57.57 $\pm$ 45.85 & 10.20 $\pm$ 12.92 \\
\bottomrule
\end{tabular}
\end{table}




\begin{wrapfigure}{r}{0.26\linewidth}
  \centering
  \vspace{-3.0em}
  \includegraphics[width=\linewidth]{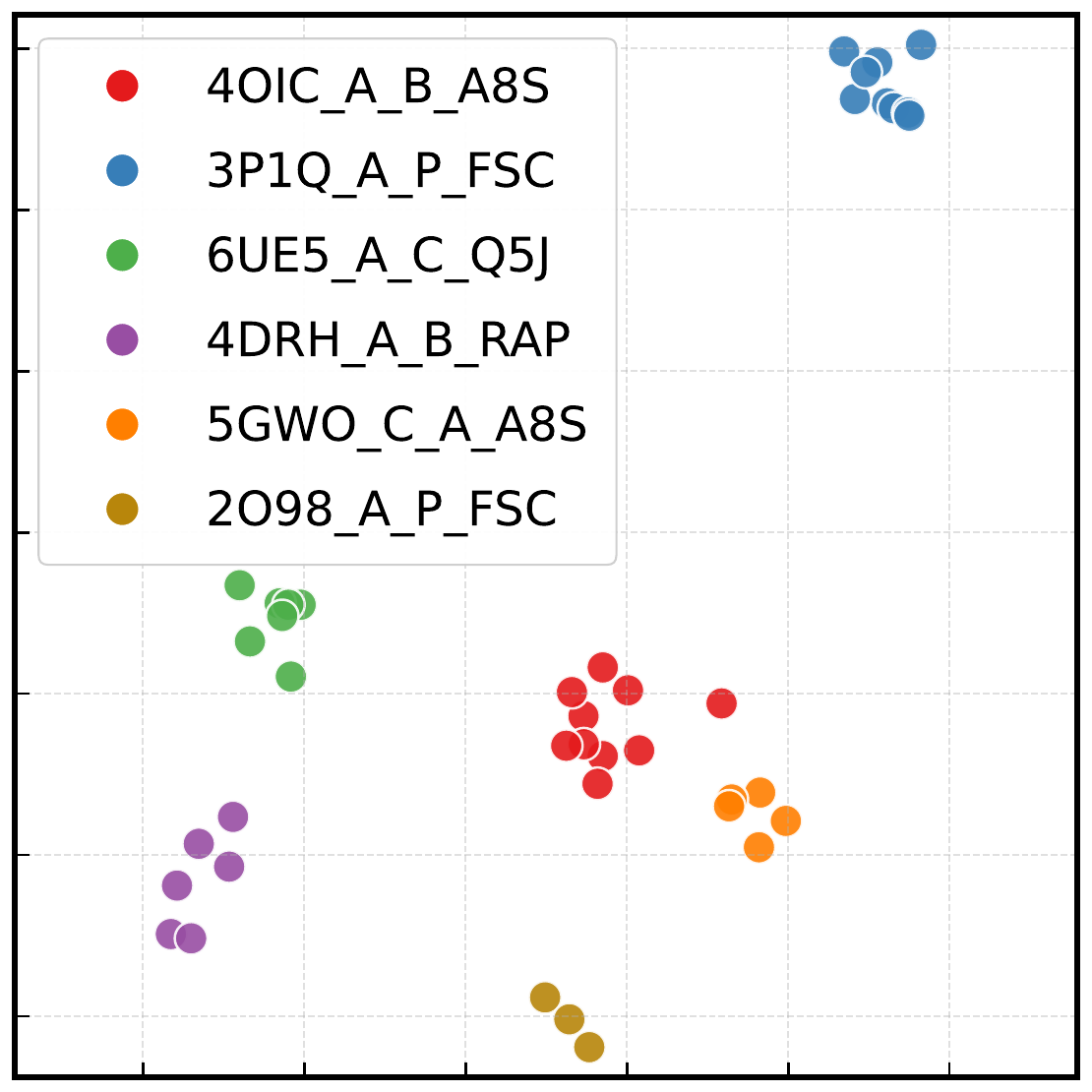}
  \caption{T-SNE of the interface model's hidden states.}
  \label{fig:interface_tsne}
  \vspace{-2.5em}
\end{wrapfigure}
\noindent \textbf{Interface Model Captures the Potential Protein-Protein Interface Pattern.} 
To further answer \textbf{RQ2}, we analyze the learned interface representations. 
We first cluster the ternary complexes in the test set by structural and functional similarity, yielding several groups represented by one or two complexes (e.g., \texttt{2O98\_A\_P\_FSC}). 
Six representative groups are selected, and the two proteins in each complex are encoded by the interface model. 
As shown in Fig.~\ref{fig:interface_tsne}, the resulting features are visualized using t-SNE~\citep{tsne}. 
Interestingly, the embeddings form compact and well-separated clusters that closely align with the structural grouping of the complexes, indicating that the model captures transferable, high-level interface geometry and interaction patterns rather than fitting samples. 

\subsection{Ablation Study}
\begin{wrapfigure}{r}{0.49\linewidth}
    \centering
    \vspace{-1.0em}
    \includegraphics[width=\linewidth]{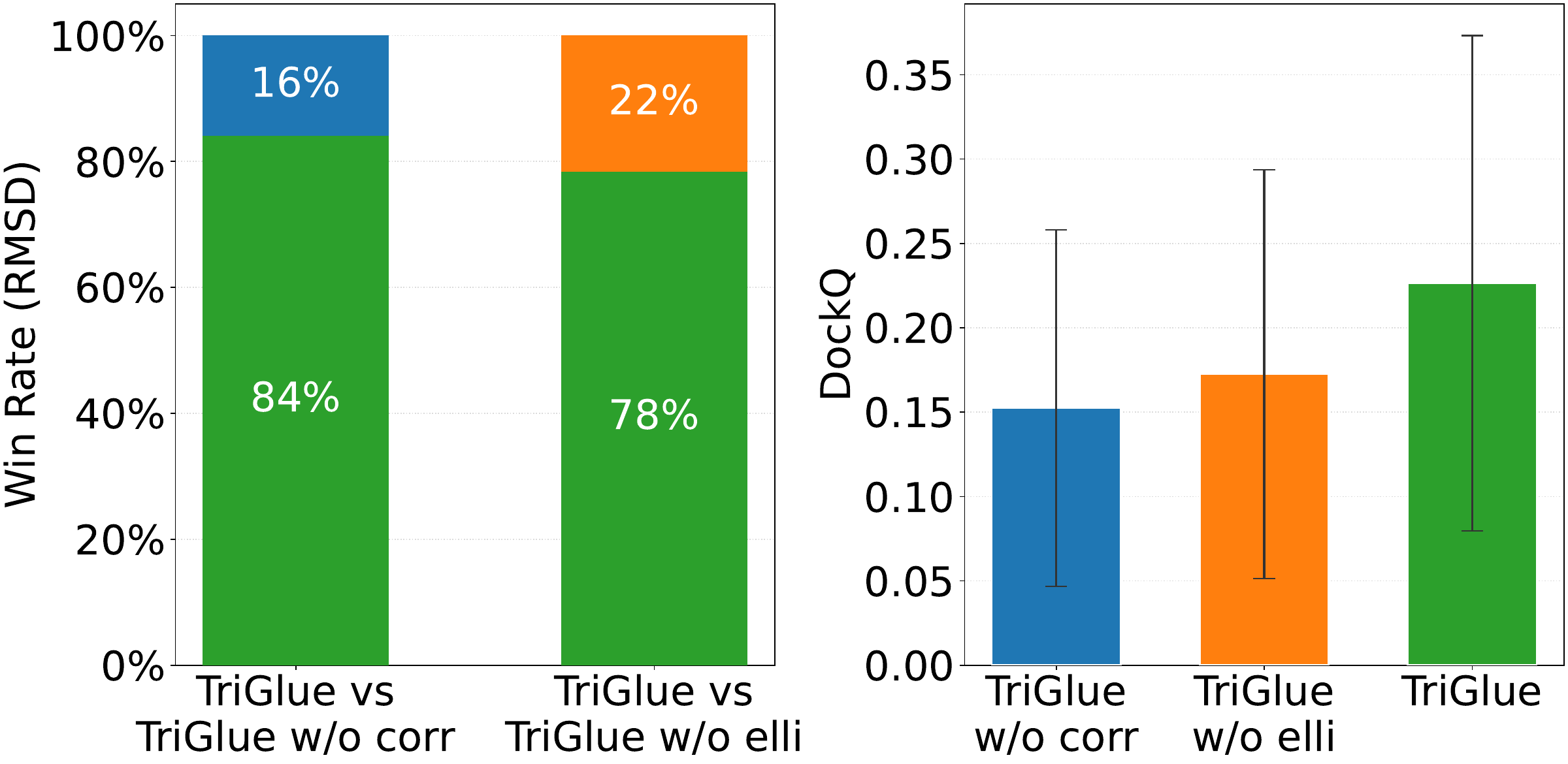}
    \caption{Ablation study.}
    \vspace{-2.0em}
    \label{fig:ablation_study}
\end{wrapfigure}
To answer \textbf{RQ3}, we conduct an in-depth ablation study to evaluate the effect of interface injection in TriGlue. 
We consider two ablated variants: 
(i) \textbf{TriGlue w/o elli}, which removes the ellipsoid interface prior (Eq.~(\ref{eq:interface1})); and 
(ii) \textbf{TriGlue w/o corr}, which disables the interface correction (Eq.~(\ref{eq:interface2})). 
We report the RMSD win rate of TriGlue against the two ablated variants, and the DockQ scores of all three models. 
As shown in Fig.~\ref{fig:ablation_study}, we observe that: 
\begin{itemize}[leftmargin=*]
    \item Both \textbf{TriGlue w/o elli} and \textbf{TriGlue w/o corr} exhibit substantially weaker docking performance for the unbound target protein, reflected by lower RMSD win rates and reduced DockQ scores. 
    This observation is consistent with our hypothesis that accurate modeling of the protein-protein interface is critical for ternary complex generation. 
    \item Notably, removing the interface correction module results in a more pronounced drop in performance. 
    The average DockQ decreases from 0.22 to 0.15, and ligand placement accuracy is also significantly affected. 
    We attribute this behavior to the increased structural noise introduced by the unconstrained rigid-body pose of the unbound protein. 
    Without interface correction, the stochasticity in protein positioning hampers the model’s ability to reliably interpret and exploit interface cues during generation.
\end{itemize}
\section{Related Work}
\textbf{Flow Matching} has emerged as an efficient paradigm for generative modeling by learning a continuous-time vector field that transports samples from a simple prior to the data distribution~\citep{flowmatching}, which demonstrates strong performance in Euclidean domains, particularly in applications such as image~\citep{flowmatchingimage}, video generation~\citep{flowmatchingvideo}, and robotics~\citep{pi0}.
Euclidean flow matching has been further generalized to non-Euclidean space, e.g., Riemannian manifolds~\citep{manifoldflowmatching}, enabling principled modeling of structured objects like protein backbone~\citep{se3flowmatchingbackbone}, and discrete DNA sequence~\citep{flowmatchingdna}.

\textbf{Deep Learning for Protein Design} has advanced rapidly following the emergence of AlphaFold~\citep{alphafold}, which revolutionized computational biology through highly accurate protein structure prediction.
This breakthrough has catalyzed a broad spectrum of generative approaches, including RFDiffusion for enzyme design~\citep{rfdiffusion,rfdiffusion2}, SE(3)-Diffusion for protein backbone generation~\citep{se3diffusionbackbone}, Boltzgen for binder design~\citep{boltzgen}, and peptide design~\citep{ppflow,pepflow,pepmimic}. 
Subsequent efforts have extended these advances to protein complexes, exemplified by AlphaMultimer~\citep{alphafoldmultimer} and AlphaFold3~\citep{alphafold3}, enabling accurate modeling of multi-chain assemblies.
However, these approaches primarily focus on pairwise interactions and do not explicitly address ternary complex design, such as molecular glue degrader or PROTAC-mediated systems. DeepTernary~\citep{deepternary} predicts ternary structures from given sequences, whereas our work targets the \emph{de novo} design of molecular glue-induced ternary complexes.
\section{Conclusion}
In this work, we investigate the problem of molecular glue design from a generative modeling perspective. 
Unlike conventional structure-based drug design, molecular glue discovery requires the coordinated modeling of ligand generation, protein--protein docking, and ternary complex assembly under unknown interface conditions. 
To address this challenge, we propose {\ours}, a biology-inspired generative framework that decomposes ternary complex generation into interface estimation and interface-conditioned complex generation. 
Specifically, an SE(3)-equivariant interface module is introduced to infer geometrically meaningful protein--protein interfaces from unbound structures, while an interface-conditioned ternary flow network jointly generates the molecular glue and predicts the rigid-body transformation required for ternary complex assembly. 
Extensive experiments demonstrate that {\ours} achieves strong performance in both molecular generation and ternary complex reconstruction, producing chemically valid molecules and structurally plausible complexes. 
Overall, our results highlight the promise of biology-inspired generative modeling for molecular glue discovery and provide a new direction for AI-driven targeted protein degradation.

\bibliography{ref}



\newpage
\appendix
\startcontents[appendix]

\section*{Appendix Contents}
\printcontents[appendix]{}{1}{\setcounter{tocdepth}{2}}

\newpage

\section{Terminology}\label{app:terminology}
Table~\ref{tab:terminology} provides supplementary introductions to the biological background and abbreviations of the technical terms that appear in this paper. 

\begin{table}[H]
\centering
\caption{Terminology list.}
\label{tab:terminology}
\renewcommand{\arraystretch}{1.20}
\begin{tabularx}{\linewidth}{
>{\raggedright\arraybackslash}p{3.2cm}
>{\centering\arraybackslash}p{1.8cm}
>{\raggedright\arraybackslash}X}

\toprule
Terminology & Abbreviation & Description \\
\midrule
\textbf{Target Protein} & \textbf{POI} & The specific protein intended to be modulated, inhibited, stabilized, or degraded by a therapeutic molecule or biological mechanism in order to achieve a desired biological or therapeutic effect. \\
\midrule
\textbf{Targeted Protein Degradation} 
& \textbf{TPD} 
& A therapeutic strategy that eliminates disease-related proteins by harnessing the cell’s natural protein disposal machinery rather than inhibiting protein activity directly. \\
\midrule
\textbf{Proteasome}
& -- 
& A large multisubunit protease complex that degrades ubiquitin-tagged proteins and maintains protein quality control in cells. \\
\midrule
\textbf{Molecular Glue Degrader} 
& \textbf{MGD} 
& A small molecule that induces interaction between an E3 ligase and a target protein, promoting ubiquitination and subsequent proteasomal degradation. \\
\midrule
\textbf{Proteolysis Targeting Chimeras} 
& \textbf{PROTACs} 
& Heterobifunctional molecules that bind a protein of interest via a target warhead and recruit an E3 ligase via a second ligand, connected by a chemical linker that controls ternary complex geometry. \\
\midrule
\textbf{E3 Ligase} 
& -- 
& An enzyme that recognizes substrate proteins and catalyzes ubiquitin transfer from E2 enzymes, determining substrate specificity in the ubiquitin proteasome system. \\
\midrule
\textbf{Ubiquitin} 
& -- 
& A conserved regulatory protein that can be attached to substrates to mark them for degradation or regulate signaling pathways. \\
\midrule
\textbf{Ubiquitination} & -- & A post-translational modification in which ubiquitin, a small regulatory protein, is covalently attached to lysine residues on a substrate protein through an enzymatic cascade involving E1, E2, and E3 enzymes, commonly marking the protein for proteasomal degradation or altering its cellular function. \\
\bottomrule
\end{tabularx}
\end{table}

Traditional drug discovery has long relied on inhibitors to modulate protein activity directly, yet this approach faces limitations when targeting certain disease-related proteins. 
Targeted protein degradation (TPD)~\citep{protacs0,mg0,lytac0,abtac0} has emerged as a transformative therapeutic paradigm, which harnesses the cell’s intrinsic proteolytic machinery, such as the proteasome and lysosomal pathways, to selectively degrade disease-associated proteins, in contrast to traditional small-molecule inhibitors that merely block protein function~\citep{tpd2,tpd1,tpd0,tpd3}.

Here, we provide additional information related to molecular glues.

\textbf{Molecular Glue} is a class of small molecules that modulate protein–protein interactions by enhancing the affinity between specific proteins within cells. 
These compounds can either induce novel protein interactions (type I) or stabilize existing ones (type II), providing an alternative therapeutic approach to traditional drug discovery. 
Notably, molecular glues offer a means to target proteins once considered ``undruggable'' by conventional strategies. 
They function through diverse mechanisms, e.g., promoting targeted protein degradation or disrupting essential protein functions, and hold potential for treating a range of diseases, including cancer and neurodegenerative disorders.

Molecular glues can be functionally classified based on their mechanism of action in modulating protein–protein interactions: type I glues stabilize non-native or novel PPIs, while type II glues enhance the stability of native, pre-existing interactions. 
\begin{itemize}
    \item \textbf{Type I} 
    Molecular glues induce non-native protein–protein interactions that effectively block or ``shield'' a protein's endogenous activity. Rather than promoting protein degradation, these compounds typically stabilize inactive conformations or obscure critical functional regions of the target protein. 
    This interference prevents the protein from participating in its normal biological processes, such as by blocking active sites, disrupting ligand binding, or inhibiting native protein–protein interactions~\citep{tpd0,dewey2023molecular}. 
    \item \textbf{Type II}
    Molecular glues stabilize endogenous protein–protein interactions by modulating protein conformation or dynamics. 
    Depending on the context, they can either inhibit or enhance protein activity by locking interacting partners into specific functional states. 
    A notable example is lenalidomide, an immunomodulatory drug that binds cereblon (CRBN) and reprograms it to target transcription factors such as IKZF1 and IKZF3 for proteasomal degradation in multiple myeloma~\citep{tpd0}. 
\end{itemize}

Generally, Molecular Glues employ two primary mechanisms to modulate protein-protein interactions: allosteric regulation and direct bridging. 
\begin{itemize}
    \item \textbf{Allosteric Regulation}.
    In the context of allosteric regulation, molecular glues interact with a primary protein, triggering conformational rearrangements that generate or reinforce novel binding interfaces. 
    These newly formed interfaces facilitate the recruitment of a secondary protein~\citep{allosteric_regulation}. 
    E.g., lenalidomide associates with the E3 ubiquitin ligase cereblon (CRBN), altering its structural configuration to enable the recruitment and subsequent ubiquitination of neo-substrates like IKZF1 and IKZF3, leading to their degradation~\citep{allosteric_regulation_example}. 
    \item \textbf{Direct Bridging}.
    Unlike allosteric modulation, direct bridging entails the molecular glue simultaneously binding to two proteins at their interface, effectively connecting them. 
    A classic example is rapamycin, which binds both FKBP12 and mTOR to form a stable ternary complex, thereby suppressing mTOR’s kinase function~\citep{rapamycin}. 
    Although such direct bridging mechanisms do occur, they are relatively rare compared to allosteric modulation. 
    The latter is more prevalent among molecular glues because it leverages the inherent flexibility of protein conformations to induce new interaction surfaces, even in the absence of pre-defined binding sites~\citep{FKBP12-mTOR}.
\end{itemize}

\section{Ellipsoid Construction}\label{app:ellip}
Given an interface residue set $\mathcal{I}$ (either $\mathcal{I}_{\mathcal{R}}$ or $\mathcal{I}_{\tilde{\mathcal{T}}}$), we approximate its spatial distribution using a Gaussian ellipsoid parameterized by a centroid $\boldsymbol{\mu}$ and covariance matrix $\boldsymbol{\Sigma}$. 
This representation provides a compact description of the interface location, spatial extent, and principal orientation in 3D space. 
Formally, the centroid is computed as the mean of the C$\alpha$ coordinates:
\begin{equation}
\boldsymbol{\mu}=
\frac{1}{|\mathcal{I}|}
\sum_{i \in \mathcal{I}} \mathbf{X}_i,
\end{equation}
and the covariance matrix is estimated as:
\begin{equation}
\boldsymbol{\Sigma}=
\frac{1}{|\mathcal{I}|}
\sum_{i \in \mathcal{I}}
(\mathbf{X}_i - \boldsymbol{\mu})
(\mathbf{X}_i - \boldsymbol{\mu})^{\top},
\end{equation}
where $\mathbf{X}_i \in \mathbb{R}^3$ denotes the C$\alpha$ coordinate of residue $i$. 

The resulting Gaussian ellipsoid $\mathcal{N}(\boldsymbol{\mu}, \boldsymbol{\Sigma})$ captures the global geometry of the interface: the mean encodes the interface center, while the eigenstructure of $\boldsymbol{\Sigma}$ reflects its spatial spread and principal axes. 
This compact parameterization serves as the geometric representation used throughout our framework for interface alignment and evaluation. 

\section{Architecture}\label{app:architecture}
\begin{algorithm}[H]
\caption{Invariant Point Attention (IPA)}
\label{alg:ipa}
\begin{algorithmic}[1]
\State \textbf{Input:} single feature $\{s_i\}$, pair features $\{z_{ij}\}$, residue frame $\{T_i\}$ 

\State $q^h_i, k^h_i, v^h_i \gets \text{LinearNoBias}(s_i)$ \Comment{$h=1,\ldots,N_{\text{head}}$}
\State $\tilde{q}^{hp}_{i}, \tilde{k}^{hp}_{i} \gets \text{LinearNoBias}(s_i)$ \Comment{$p=1,\ldots,N_{\text{query points}}$}
\State $\tilde{v}^{hp}_{i} \gets \text{LinearNoBias}(s_i)$ \Comment{$p=1,\ldots,N_{\text{point values}}$}
\State $b^h_{ij} \gets \text{LinearNoBias}(z_{ij})$

\State $w_C \gets \sqrt{\tfrac{2}{9N_{\text{query}}}}, \quad w_L \gets \sqrt{\tfrac{1}{c}}$
\State $a^h_{ij} \gets \mathrm{softmax}_j \Big( w_L \cdot (q^h_i)^\top k^h_j - 
    \tfrac{w_C^2}{2} \sum_p \| T_i \tilde{q}^h_{i,p} - T_j \tilde{k}^h_{j,p} \|_2^2 \Big)$

\State $\tilde{o}^h_i \gets \sum_j a^h_{ij} z_{ij}$
\State $o^h_i \gets \sum_j a^h_{ij} v^h_j$
\State $\tilde{o}^{h,p}_i \gets T_i^{-1} \Big( \sum_j a^h_{ij} \, T_j \tilde{v}^h_{j,p} \Big)$

\State $\tilde{s}_i \gets \text{Linear}\Big( \text{concat}_{h,p}(\tilde{o}^h_i, o^h_i, \tilde{o}^{h,p}_i, \|\tilde{o}^{h,p}_i\|) \Big)$

\State \Return $\{\tilde{s}_i\}$
\end{algorithmic}
\end{algorithm}

In this section, we provide the implementation details of our neural network architecture. 

\noindent \textbf{Latent Ternary Encoder}. 
We adopt Invariant Point Attention (IPA)~\citep{alphafold} to encode the proteins and ligand in the ternary flow network. 
As shown in Algorithm~\ref{alg:ipa},
the IPA block jointly attends over scalar features, pairwise relational features, and 3D coordinates. 
The key design principle is SE(3) invariance: the model’s predictions remain unchanged under global rotations and translations of the input structure. 
IPA augments attention with geometric queries and keys represented as points in 3D space. 
During attention computation, the similarity between residues depends not only on feature similarity but also on the relative spatial configuration between these points. The resulting attention weights therefore encode both biochemical context and geometric compatibility, allowing the network to iteratively refine spatial representations. 

\textbf{Node (single) features.}
For each protein's residue, we form the single feature by concatenating: (i) a learned amino-acid embedding; (ii) heavy-atom coordinates expressed in the local frame defined by the N--C$\alpha$--C; and (iii) backbone dihedral descriptors with chain-aware validity masking. 
For the ligand, each atom is embedded by a learned atom-type embedding as the single feature. 
The ligand does not have the residue frame. 
So we set the rotation matrix as the identity matrix and the translation as the atom's coordinate to construct the local coordinate system. 

\textbf{Pair (edge) features.}
For protein residue pairs $(i,j)$, the pair feature encodes: (i) a learned embedding of the ordered amino-acid pair; (ii) sequence separation along the same chain (clipped integer offset, zeroed across chains); (iii) all-pairs heavy-atom distances summarized by Gaussian radial-basis activations whose scale is conditioned on the amino-acid pair; and (iv) pairwise backbone dihedral channels. 
For ligand atom pairs $(i,j)$, we use Gaussian-bump encodings of the interatomic distance together with an MLP on the normalized displacement $\mathbf{r}_{ij}/\|\mathbf{r}_{ij}\|$.

All of the single (pair) features are mapped into a unified dimension. 
Given the input $\mathcal{G}_{\mathcal{R}}$, $\mathcal{G}_{\tilde{\mathcal{T}}}$, and $\mathcal{G}_{\tilde{\mathcal{M}}}$ (noised graph from $\mathcal{G}_{\mathcal{M}}$), we encode them with the \emph{Invariant Point Attention} block~\citep{alphafold}:
\begin{equation}
    \mathbf{H}_{\mathcal{R}}, \mathbf{H}_{\tilde{\mathcal{T}}},\mathbf{h}_{\tilde{\mathcal{M}}}=
    \mathrm{IPA}(\mathcal{G}_{\mathcal{R}},\mathcal{G}_{\tilde{\mathcal{T}}},\mathcal{G}_{\tilde{\mathcal{M}}}),
\end{equation}
where $\mathbf{H}_{\mathcal{R}}, \mathbf{H}_{\tilde{\mathcal{T}}},~ \text{and}~\mathbf{h}_{\tilde{\mathcal{M}}}$ denote invariant single representations that encode both sequence and structural information, which are subsequently used by the denoising networks. 

\noindent \textbf{Rigid Transformation Block}. 
The $\Phi_{\text{RT}}$ is designed as a lightweight global network that predicts an incremental rigid transformation from a compact interface-level representation. 
At each layer $i$, $\Phi_{\text{RT}_i}$ uses a fixed-length global descriptor that summarizes the current alignment state and noised context, and predicts the residual rotation vector and translation vector. 

Specifically, given the virtual interface point sets $\mathbf Y_{\mathcal R}$ and $\mathbf Y_{\tilde{\mathcal T}}$, we first compute their centroids:
\begin{equation}
\boldsymbol\mu_1=\frac{1}{K}\sum_{k=1}^K \mathbf y_{\mathcal R,k},
\qquad
\boldsymbol\mu_2=\frac{1}{K}\sum_{k=1}^K \mathbf y_{\tilde{\mathcal T},k}.
\end{equation}
We then construct a global rigid-state descriptor by concatenating:
\begin{equation}
H =
\text{Concat}\!\left(\mathbf R^{*},\mathbf t^{*},
\boldsymbol\mu_1,\boldsymbol\mu_2,
\mathbf R_t,\mathbf t_t,
\mathbf c_t
\right),
\end{equation}
where $\mathbf c_t$ is a low-dimensional context vector obtained by projecting the timestep embedding. 
The descriptor $h$ is processed by a shared multi-layer perceptron trunk
\begin{equation}
h = \text{MLP}_{\text{trunk}}(H),
\end{equation}
which produces a latent global representation capturing the relative pose between the proteins and the timestep. 
Two linear heads predict rotation and translation updates:
\begin{align}
\Delta\omega &= \text{RotHead}(h),\\
\Delta\mathbf t &= \text{TransHead}(h),\\
     \mathbf{R}^{*} &= \exp([\Delta\omega\times])\mathbf{R}^{*},\\
     \mathbf{t}^{*} &= \mathbf{t}^{*} + \Delta\mathbf{t},\\
     \mathbf{Y}_{\mathcal{\tilde{\mathcal{T}}}} &= \mathbf{R}^{*}\mathbf{Y}_{\mathcal{\tilde{\mathcal{T}}}} + \mathbf{t}^{*},
\end{align}
Notably, rotation is parameterized using the continuous 6D representation and converted to a valid rotation matrix via Gram--Schmidt orthogonalization. 
Finally, after $T$ iterations, block $\Phi_{\text{RT}}$ outputs the prediction $\hat{\mathbf{R}}_1,\hat{\mathbf{t}}_1$.

\noindent \textbf{Coordinate Block}.
The $\Phi_{\mathbf X}$ refines ligand atom positions that jointly model intra-ligand geometry and protein–ligand interactions. 
Given ligand coordinates $\mathbf X_t=\{\mathbf x_{i,t}\}_{i=1}^{N}$ at step $t$, the block predicts denoised coordinates through iterative message passing. 
Ligand atom features are first projected into a latent space with an IPA block and enriched with protein context using cross-attention. 
Specifically, ligand single features attend to the concatenated single features of the two proteins, allowing each atom to capture biochemical context:
\begin{equation}
\mathbf h_i = \text{CrossAttn}(\mathbf h_i,\, \mathbf H_{\mathcal{R}}, \mathbf H_{\tilde{\mathcal{T}}}).
\end{equation}

The block performs $L$ EGNN-style updates. 
At each layer, ligand features are updated using messages from three interaction types:
ligand–ligand, ligand– receptor protein, ligand–target protein. 

For any pair of nodes $(i,j)$, messages depend only on invariant features:
\[
\mathbf{m}_{ij} = \phi\big(\mathbf h_i,\mathbf h_j,\|\mathbf x_{i,t}-\mathbf x_{j,t}\|^2,t\big),
\]
and are aggregated via residual updates to produce refined atom embeddings.

After each feature update, coordinates are refined using a bi-level message passing with the messages:
\begin{align}
\hat{\mathbf x}_{i} &=
\mathbf x_{i,t}
+
\underbrace{
\sum_{j\in\mathcal V_I}
(\mathbf x_{i,t}-\mathbf x_{j,t})
\phi_{\text{intra}}
(\mathbf h_i,\mathbf h_j,\|\mathbf x_{i,t}-\mathbf x_{j,t}\|^2,t)
}_{\text{Intrinsic ligand update}}
\nonumber\\
&\quad+
\underbrace{
\sum_{k\in\mathcal V_E}
(\mathbf x_{i,t}-\mathbf X_k)
\phi_{\text{extra}}
(\mathbf h_i,\mathbf H_k,\|\mathbf x_{i,t}-\mathbf X_k\|^2,t)
}_{\text{Protein-conditioned update}},
\end{align}
where $\phi(\cdot)$ is MLP, and only ligand coordinates are updated. 
Feature and coordinate updates are repeated for multiple layers, yielding a coarse-to-fine denoising process. 
Finally, the $\Phi_{\mathbf{x}}$ block outputs the predicted coordinate $\hat{\mathbf{x}_1}$.

\noindent \textbf{Atom Type Block.} 
To predict the atom type, we first form a joint protein context from single representations:
\begin{equation}
\mathbf{H}_{\mathcal{P}} = 
\text{Concat}\!\left(\mathbf{H}_{\mathcal{R}}, \mathbf{H}_{\tilde{\mathcal{T}}}\right).
\end{equation}
Then, we update the ligand hidden state via cross-attention:
\begin{equation}
\mathbf{h}_{\tilde{\mathcal{M}}}^{*}
= \text{CrossAttn}\!\left(
\mathbf{h}_{\tilde{\mathcal{M}}},
\mathbf{H}_{\mathcal{P}}
\right).
\end{equation}
Finally, the atom-type distribution is predicted with timestep embedding $\mathbf{c}_t$:
\begin{equation}
\hat{\mathbf{a}}_{1}
= \Phi_{\mathbf{a}}\!\left(
\mathbf{h}_{\tilde{\mathcal{M}}}^{*},\,
\mathbf{c}_t
\right).
\end{equation}
We sample the atom type from the distribution to decode the ligand.

\section{Flow Matching}\label{app:flow_matching}
\noindent \textbf{Conditional Flow Matching}. The conditional flow matching (CFM) framework \citep{flowmatching} offers an approach to learn a probability time-dependent flow $\psi_t : [0, 1] \times \mathbb{R}^d \rightarrow \mathbb{R}^d$ that transforms samples from a simple base distribution $p_0$ into a target data distribution $p_1 \approx q$. 
The flow and its corresponding vector field $u_t : [0, 1] \times \mathbb{R}^d \rightarrow \mathbb{R}^d$ are governed by the ordinary differential equation (ODE): $\frac{\text{d}}{\text{d}t} x_t = u_t(x_t)$, where $x_t = \psi_t(x_0)$. 
However, it is challenging to model this vector field directly due to the complexity of its closed-form solution. 
Instead, CFM circumvents this limitation by conditioning the flow and vector field on a specific data pair: a sample $x_0 \sim p_0$ and a target $x_1 \sim p_1$. 
This leads to a tractable formulation of the interpolated trajectory $x_t = \psi_t(x_0 \mid x_1)$ and the conditional vector field $u_t(x_t \mid x_1) = \frac{\text{d}}{\text{d}t} x_t$. 
Hence, the standard CFM training loss is formulated as:
\begin{equation}\label{eq:cfm_loss}
    \mathcal{L}_{\text{CFM}}(\theta) = \mathbb{E}_{t \sim \mathcal{U}(0, 1)} \left\| v_{t}(\psi_{t}(x_{0} \mid x_1)) - \frac{\text{d}}{\text{d}t}x_t  \right\|_2^2,
\end{equation}
where $\mathcal{U}(0, 1)$ is a uniform distribution and $v_{t}$ is a neural network parameterized by $\theta$. 

Based on the CFM, our Ternary Flow Network involves three types of flow:

\noindent \textbf{Euclidean CFM for Coordinate.}
The Euclidean conditional flow interpolates between Gaussian noise $\mathbf{x}_{i,0}\sim\mathcal N(\mathbf 0,I_3)$ and the data distribution $\mathbf{x}_{i,1}\sim p(\mathbf x_i\mid\star)$:
\begin{equation}
\psi_t(\mathbf{x}_{i,0}\mid\mathbf{x}_{i,1})
= t\mathbf{x}_{i,1}+(1-t)\mathbf{x}_{i,0}
\end{equation}
The target vector field is:
\begin{equation}
u_t
= \frac{d}{dt}\psi_t
= \mathbf{x}_{i,1}-\mathbf{x}_{i,0}.
\end{equation}
The CFM loss is:
\begin{equation}\label{eq:euclidean_cfm}
\mathcal L_{\text{coord}}
=\mathbb E_{t\sim U(0,1)}
\left\|
v_\theta(\mathbf x_{i,t},t,\star)
-(\mathbf x_{i,1}-\mathbf x_{i,0})
\right\|_2^2.
\end{equation}

\noindent \textbf{Simplex CFM for Atom Type}. 
Following PepFlow~\citep{pepflow}, we map the discrete atom type $\mathbf{a}_i$ to a continuous logit space $\mathbf{s}_i \in \mathbb{R}^{\mid\mathcal{A}\mid}$ via soft one-hot encoding with a scale constant $K > 0$:
\begin{equation}
    \mathbf{s}_i[j] = 
    \begin{cases}
        K, & j = \mathbf{a}_i, \\
        -K, & \text{otherwise}.
    \end{cases}
\end{equation}
This yields a sharp probability distribution \texttt{softmax}($\mathbf{s}_i$) concentrated on the true type, corresponding to a point on the probability simplex $\Delta^{\mid\mathcal{A}\mid - 1}$.
We define the conditional flow in the logit space as a linear interpolation between the prior $\mathbf{s}_{i,0} \sim \mathcal{N}(\mathbf{0}, K^2 I_{\mid\mathcal{A}\mid})$ and the target $\mathbf{s}_{i,1}$ derived from $\textbf{a}_i$:
\begin{equation}
    \boldsymbol{\psi}_{t}(\mathbf{s}_{i,0} \mid \mathbf{s}_{i,1}) = t \mathbf{s}_{i,1} + (1 - t) \mathbf{s}_{i,0}.
\end{equation}
The associated vector field is:
\begin{equation}
    u_{t}(\mathbf{s}_{i,t} \mid \mathbf{s}_{i,1}, \mathbf{s}_{i,0}) = \frac{\mathrm{d}}{\mathrm{d}t} \boldsymbol{\psi}_t = \mathbf{s}_{i,1} - \mathbf{s}_{i,0}.
\end{equation}
Similar to the Euclidean CFM, a neural network $v_{\theta}$ is trained to predict this vector field, and the CFM loss is:
\begin{equation}\label{eq:simplex_cfm}
    \mathcal{L}_{\text{type}} = \mathbb{E}_{t \sim \mathcal{U}(0,1)} \left[ \left\| v_{\theta}(\mathbf{s}_{i,t}, t, \star) - (\mathbf{s}_{i,1} - \mathbf{s}_{i,0}) \right\|_2^2 \right].
\end{equation}

\noindent \textbf{Riemannian CFM for Rigid Transformation}. 
Following \citep{se3flowmatchingbackbone}, we predict $(R, \mathbf{t})$ within the flow matching framework, where the interpolation of orientations is performed along geodesics on $\mathrm{SO}(3)$, and translations are handled in Euclidean space. 

For rotation, we employ geodesic interpolation on $\mathrm{SO}(3)$, which serves as the manifold counterpart of linear interpolation in Euclidean space and guarantees shortest-path evolution.
Formally, given a random prior $\mathbf{R}_0 \sim U(\mathrm{SO}(3))$ and a target rotation $\mathbf{R}_1 \sim p(\mathbf{R} \mid \star)$, the conditional flow is:
\begin{equation}
    \boldsymbol{\psi}_{t}(\mathbf{R}_0 \mid \mathbf{R}_1) 
    = \exp_{\mathbf{R}_0}\!\big(t \cdot \log_{\mathbf{R}_0}(\mathbf{R}_1)\big),
\end{equation}
where $\exp$ and $\log$ denote the Riemannian exponential and logarithm maps on $\mathrm{SO}(3)$, computable via Rodrigues’ formula. 
The associated vector field is defined as:
\begin{equation}
    u_{t}(\mathbf{R}_t \mid \mathbf{R}_0, \mathbf{R}_1) 
    = \frac{\mathrm{d}}{\mathrm{d}t}\boldsymbol{\psi}_{t}
    = \frac{\log_{\mathbf{R}_t}(\mathbf{R}_1)}{1-t}.
\end{equation}
Then the CFM objective is formulated as:
\begin{equation}\label{eq:rotation_cfm}
    \mathcal{L} = 
    \mathbb{E}_{t \sim \mathcal{U}(0,1)}\left[
    \left\| v_{\theta}(\mathbf{R}_t, t, \star) - \frac{\log_{\mathbf{R}_t}(\mathbf{R}_1)}{1-t} \right\|_2^2\right].
\end{equation}

For translations, a Euclidean conditional flow is defined between $\mathbf{t}_0 \sim \mathcal{N}(\mathbf{0}, \sigma^2 I)$ and $\mathbf{t}_1 \sim p(\mathbf{t} \mid \star)$, and the corresponding CFM obejctive $\mathcal{L}$ has the same form with Eq.~(\ref{eq:euclidean_cfm}). 

\noindent \textbf{Reparameterized CFM Objectives.}
Following the reparameterization strategy~\citep{pepflow}, instead of predicting the ground-truth conditional vector field, we directly predict the target data and compute the vector field implicitly from the reconstructed target.
For the Euclidean flow, let $\hat{\mathbf x}_{i,1}=v_\theta(\mathbf x_{i,t},t,\star)$ denote the reconstructed ligand coordinate, we reparameterized the objective in Eq.~(\ref{eq:euclidean_cfm}) as:
\begin{align}
\mathcal L_{\text{coord}}
&=\mathbb E_{t \sim \mathcal{U}(0,1)}
\left\|
u_t(\mathbf x_{i,t}\mid\hat{\mathbf x}_{i,1},\mathbf x_{i,0})
-
u_t(\mathbf x_{i,t}\mid\mathbf x_{i,1},\mathbf x_{i,0})
\right\|_2^2 \\
&=\mathbb E_{t \sim \mathcal{U}(0,1)}
\left\|
(\hat{\mathbf x}_{i,1}-\mathbf x_{i,0})
-(\mathbf x_{i,1}-\mathbf x_{i,0})
\right\|_2^2 \\
&=\mathbb E_{t \sim \mathcal{U}(0,1)}
\left\|\hat{\mathbf x}_{i,1}-\mathbf x_{i,1}\right\|_2^2 .
\end{align}
Similarly, the other CFM objectives are reformulated as follows:
\begin{align}
& \mathcal L_{\text{type}}
=\mathbb E_{t \sim \mathcal{U}(0,1)}
\left\|
\hat{\mathbf s}_{i,1}-\mathbf s_{i,1}
\right\|_2^2 . \\
& \mathcal L_{\text{rot}}
=\mathbb E_{t \sim \mathcal{U}(0,1)}
\left\|
\left(\log_{\mathbf R_t}(\hat{\mathbf R}_1)
-
\log_{\mathbf R_t}(\mathbf{R}_1)\right) /
(1-t)
\right\|_2^2 . \\
& \mathcal L_{\text{trans}}
=\mathbb E_{t \sim \mathcal{U}(0,1)}
\left\|
\hat{\mathbf t}_1-\mathbf t_1
\right\|_2^2 .
\end{align}
In this work, we utilize the reparameterized CFM objectives to train our model.

\section{Sample Algorithm}\label{app:sample}

\begin{algorithm}[H]
\caption{Sampling with Interface-Conditioned Ternary Flow Network}
\label{alg:sampling}
\begin{algorithmic}[1]

\State \textbf{Input:} receptor protein $\mathcal{G}_\mathcal{R}$, target protein $\mathcal{G}_\mathcal{\tilde{T}}$.
\vspace{0.2em}

\State $(\mathbf{Y}_\mathcal{R},\mathbf{Y}_\mathcal{\tilde{T}})
=
\text{InterfaceModel}(\mathcal{G}_\mathcal{R},\mathcal{G}_\mathcal{\tilde{T}})$ 
\Comment{Estimate global virtual interface }

\State $(\mathbf{R}^\star,\textbf{t}^\star)=\text{DifferentiableKabschAlign}(\mathbf{Y}_\mathcal{R},\mathbf{Y}_\mathcal{\tilde{T}})$
\Comment{Compute first-order rigid transformation}

\State $\hat{\mathbf{X}}_{\mathcal{T}} = \mathbf{R}^{*}\mathbf{X}_{\tilde{\mathcal{T}}} + \mathbf{t}^{*}$
\Comment{Interface correction}

\State Sample initial noisy state $\mathbf{R}_0,\mathbf{t}_{0},\mathbf{x}_{0},\mathbf{a}_{0}$

\For{$t=1$ to $N$}

\State
$\hat{\mathbf{R}}_{\frac{t}{N}},\hat{\mathbf{t}}_{\frac{t}{N}},\hat{\mathbf{x}}_{\frac{t}{N}},\hat{\mathbf{a}}_{\frac{t}{N}}=\text{TernaryFlowNet}(\mathbf{R}_{\frac{t-1}{N}},\mathbf{t}_{\frac{t-1}{N}},\mathbf{x}_{\frac{t-1}{N}},\mathbf{a}_{\frac{t-1}{N}},\mathbf{Y}_\mathcal{R},\mathbf{Y}_\mathcal{\tilde{T}})$

\State
$\mathbf{R}_{\frac{t}{N}}
=
\exp_{\mathbf{R}_{\frac{t-1}{N}}}
\Big(
\Delta t \cdot
\log_{\mathbf{R}_{\frac{t-1}{N}}}
(
\hat{\mathbf{R}}_{\frac{t}{N}}
)
\Big)$
\Comment{Update rotation}

\State 
$\mathbf{t}_{\frac{t}{N}}=\mathbf{t}_{\frac{t-1}{N}}+(\hat{\mathbf{t}}_{\frac{t}{N}}-\mathbf{t}_0)\Delta{t}$
\Comment{Update translation}

\State $\mathbf{x}_{\frac{t}{N}}
=
\mathbf{x}_{\frac{t-1}{N}}
+
(
\hat{\mathbf{x}}_{{\frac{t}{N}}}
-
\mathbf{x}_{0}
)\Delta t$
\Comment{Update coordinate}

\State $\mathbf{a}_{\frac{t}{N}}
=
\mathbf{a}_{\frac{t-1}{N}}
+
(
\hat{\mathbf{a}}_{\frac{t}{N}}
-
\mathbf{a}_{0}
)\Delta t$
\Comment{Update atom type}
\EndFor

\State
Decode complex $\mathcal{C}$ from $\{\hat{\mathbf{R}}_1,\hat{\mathbf{t}}_1,\hat{\mathbf{x}}_1,\hat{\mathbf{a}}_1\}$

\State \textbf{return} $\mathcal{C}$

\end{algorithmic}
\end{algorithm}

Algorithm~\ref{alg:sampling} summarizes the sampling procedure of {\ours}. 
Given two unbound proteins, the interface model first predicts the global virtual interfaces and computes an initial rigid-body transformation through differentiable Kabsch alignment. 
The estimated transformation is then used to correct the target protein coordinates, providing an interface-consistent initialization for subsequent generation. 

Starting from noisy rotations, translations, ligand coordinates, and atom types, the ternary flow network iteratively predicts the denoised states conditioned on the estimated interfaces. 
The rotational component is updated along the $\mathrm{SO}(3)$ geodesic, while translations, ligand coordinates, and atom types are progressively refined using Euler steps. 
After $N$ iterations, the final rigid transformation and ligand structure are decoded into the ternary complex.

\section{Proofs}\label{app:proof_virtual_interface}
\noindent \textbf{SE(3) Equivariance of Virtual Interface and Ellipsoid Parameters.}
\renewcommand{\thetheorem}{\arabic{theorem}}
\begin{theorem}
The virtual interface points $\{\mathbf y_k\}_{k=1}^K$ and the elliposid parameters 
$(\hat{\boldsymbol\mu},\hat{\boldsymbol\Sigma})$ defined in Sec.~\ref{sec:interface_predict} are SE(3)-equivariant.
\end{theorem}

\begin{proof}
Consider a rigid-body transformation $(\mathbf R,\mathbf t)\in \text{SE(3)}$ applied to all residue coordinates:
\begin{equation}
\mathbf z_i'=\mathbf R\mathbf z_i+\mathbf t.
\end{equation}
We first prove that the virtual interface points are equivariant. 
Since the attention weights are computed from EGNN invariant features, they remain unchanged under SE(3). 
Thus, we have:
\begin{align}
\mathbf y_k'
&=\sum_i \alpha_i^k \mathbf z_i'
=\sum_i \alpha_i^k (\mathbf R\mathbf z_i+\mathbf t) \\
&=\mathbf R\sum_i \alpha_i^k\mathbf z_i + \mathbf t
=\mathbf R\mathbf y_k+\mathbf t.
\end{align}
Hence $\mathbf y_k$ are SE(3)-equivariant. 
Then, it is not hard to obtain that:
\begin{equation}
\hat{\boldsymbol\mu}'
=\frac{1}{K}\sum_k \mathbf y_k'
=\frac{1}{K}\sum_k (\mathbf R\mathbf y_k+\mathbf t)
=\mathbf R\hat{\boldsymbol\mu}+\mathbf t,
\end{equation}
and the centered coordinates satisfy:
\begin{equation}
\mathbf y_k'-\hat{\boldsymbol\mu}'=\mathbf R(\mathbf y_k-\hat{\boldsymbol\mu}),
\end{equation}
which yields:
\begin{align}
\hat{\boldsymbol\Sigma}'
&=\frac{1}{K}\sum_k (\mathbf y_k'-\hat{\boldsymbol\mu}')
(\mathbf y_k'-\hat{\boldsymbol\mu}')^\top \\
&=\mathbf R\hat{\boldsymbol\Sigma}\mathbf R^\top.
\end{align}
Therefore, $(\mathbf y_k,\hat{\boldsymbol\mu},\hat{\boldsymbol\Sigma})$ transform equivariantly under SE(3).
\end{proof}

\section{Experimental Details}\label{app:exp_details}

\subsection{Dataset Details}\label{app:dataset_details}
We use the molecular-glue-induced ternary-complex split derived from TernaryDB\footnote{\url{https://doi.org/10.5281/zenodo.15514874}}~\citep{deepternary}. 
Each sample contains two protein components, including a target protein and an E3 ligase, and one small-molecule ligand extracted from an experimentally resolved ternary complex. 
Protein lengths are calculated based on the number of resolved C$\alpha$ residues in the curated protein PDB files. 
Ligand descriptors are computed from the per-complex SDF files using RDKit. 
QED denotes the quantitative estimate of drug-likeness. Detailed statistics are provided in Tab.~\ref{tab:dataset_summary}.
For numerical descriptors, we report the median, with the minimum and maximum values. 

\begin{table}[H]
\centering
\caption{Summary statistics of the training and held-out molecular-glue test splits. Numerical descriptors are reported as \textbf{Median (Minimum, Maximum)}. 
}
\label{tab:dataset_summary}
\renewcommand{\arraystretch}{1.12}
\setlength{\tabcolsep}{4pt}
\begin{adjustbox}{max width=\linewidth}
\begin{tabular}{llcc}
\toprule
\textbf{Component} & \textbf{Statistic} & \textbf{Training set} & \textbf{Test set} \\
\midrule
\multirow{2}{*}{\textit{Complex}} & Number of ternary complexes & 21,481 & 94 \\
 & \textbf{Sequence-pair overlap with training set} & -- & \textbf{0/94} \\
\midrule
\multirow{2}{*}{\textit{Protein}} & Length of E3 ligase & 321 (3,~1765) & 227 (107,~826) \\
 & Length of target protein & 125 (3,~1759) & 76 (2,~512) \\
\midrule
\multirow{5}{*}{\textit{Molecular Glue}} & Number of heavy atoms & 32 (2,~100) & 25 (10,~69) \\
 & Molecular weight (Da) & 506.20 (26.04,~1464.07) & 385.85 (136.15,~980.31) \\
 & Number of rings & 1 (0,~24) & 3 (0,~13) \\
 & Number of aromatic rings & 0 (0,~7) & 0 (0,~5) \\
 & QED & 0.24 (0.02,~0.93) & 0.54 (0.12,~0.88) \\
\bottomrule
\end{tabular}
\end{adjustbox}
\end{table}

\subsection{Baseline Details}\label{app:baseline}
\noindent \textbf{Baselines}. 
Our framework jointly performs \emph{de novo} molecular glue generation and reconstructs the unbound target protein into the ternary-complex coordinate frame. 
However, no existing method evaluates both aspects simultaneously. 
Therefore, we select and compare representative baselines separately for each task. 
Overall, these baselines span a broad spectrum of representative methods:
\begin{itemize}[leftmargin=*,label={}]
    \item (i)~\emph{Molecular design method.} 
    To assess the quality of generated molecular glues, three structure-based drug design methods are considered: 
    \textbf{AR}\footnote{\url{https://github.com/luost26/3D-Generative-SBDD}}~\citep{ar}, 
    \textbf{TargetDiff}\footnote{\url{https://github.com/guanjq/targetdiff}}~\citep{targetdiff}, and 
    \textbf{PocketXMol}\footnote{\url{https://github.com/pengxingang/PocketXMol}}~\citep{pocketxmol}, 
    which perform \textit{de novo} small-molecule generation conditioned on protein binding pockets.

    \item (ii)~\emph{Ternary complex structure prediction method.} 
    \textbf{DeepTernary}\footnote{\url{https://github.com/youqingxiaozhua/DeepTernary}}~\citep{deepternary} 
    is the first model designed for ternary complex structure prediction. 
    We adopt it to dock the unbound protein back into the ternary complex, thereby evaluating the ability to model the complex structure. 

    \item (iii)~\emph{Protein-protein docking method.} 
    To further evaluate the rigid-body docking capability, we include two representative protein-protein docking approaches: 
    \textbf{EquiDock}\footnote{\url{https://github.com/octavian-ganea/equidock_public}}~\citep{equidock} and 
    \textbf{DiffDock-PP}\footnote{\url{https://github.com/ketatam/DiffDock-PP}}~\citep{diffdockpp}, 
    which predict the relative pose of interacting proteins from their unbound structures. 
\end{itemize}
All of the baselines are implemented with the official code. 

\subsection{Training Details}\label{app:train_details}

The proposed framework consists of two stages: interface estimation and interface-conditioned ternary flow generation. 
All experiments were conducted using the Adam optimizer with an initial learning rate of $5\times10^{-4}$.

\noindent \textbf{Interface Estimation Module}. 
The interface estimation model was trained for 1200 epochs with a batch size of 8 using distributed training on 8 GPUs. 
Node and edge embedding dimensions were set to 128 and 64, respectively. 
The SE(3)-equivariant encoder adopted a two-block IPA architecture with 8 attention heads, 8 query/key points, and 12 value points. 
A lightweight sequence transformer with 2 layers and 4 attention heads was further employed for sequence feature integration. 
The interface conditioner used a feature dimension of 128 and graph construction with 16 nearest neighbors.

\noindent \textbf{Interface-Conditioned Ternary Flow Network}. 
The ternary flow model was trained for 800 epochs with a batch size of 2 and gradient accumulation of 32, resulting in an effective batch size of approximately 64 samples per optimization step. 
The pretrained interface estimation module was frozen during this stage. 
The coordinate, atom-type, rotation, and translation objectives were jointly optimized with loss weights of $10.0$, $1.0$, $1.0$, and $0.008$, respectively. 
Gradient clipping with a maximum norm of 50 was applied to stabilize training.

For flow matching, translations were modeled with Euclidean conditional flow matching using Gaussian noise with $\sigma=1.0$, while rotations were modeled on $\mathrm{SO}(3)$ using Riemannian flow matching. 
Atom types were represented on the probability simplex with simplex scale constant $K=5.0$ and 25 atom categories, where atom identities were defined based on elemental type together with hybridization and aromaticity information. 
During sampling, the model performed 50 iterative denoising steps. 

\noindent \textbf{Dataset Split}. 
The ternary complex dataset was divided into training and validation subsets using a 95\%/5\% split. 
Unless otherwise specified, all experiments used a random seed of 42. 

Training the interface model takes approximately 2 days, training the ternary flow network requires about 5 days, and sampling a single complex takes around 40 seconds. 

\subsection{Evaluation Metric Details}\label{app:metric}
We evaluate the models from the following two perspectives:

(i)~\textbf{Molecular generation}. 
We assess the generated ligands with respect to binding and physicochemical properties. 
Binding affinity is estimated using AutoDock Vina~\citep{autodock}, including the \textbf{Vina score}, the minimized Vina score (\textbf{Vina min}) after local relaxation, and the redocking Vina score (\textbf{Vina dock}) that reflects the best achievable binding affinity. 
The high-affinity ratio (\textbf{Aff}) is defined as the proportion of generated molecules whose Vina dock surpasses that of the reference ligand for each target. 
To further examine chemical quality, we evaluate standard molecular property metrics, including quantitative estimation of drug-likeness (\textbf{QED}) and the synthetic accessibility (\textbf{SA}) score.

(ii)~\textbf{Ternary-complex reconstruction}. 
We evaluate the structural accuracy of the generated complex from both ligand and protein perspectives. 
Ligand pose accuracy is measured using the root-mean-square deviation (\textbf{RMSD}) between predicted and reference ligand coordinates. 
Protein docking accuracy is assessed using the \textbf{DockQ} score~\citep{dockq}, and we additionally report the success rate with DockQ greater than 0.23, which is commonly regarded as indicating acceptable docking quality. 

\noindent \textbf{Root Mean Square Deviation (RMSD)}. 
Ligand pose quality is assessed using the RMSD between the predicted ligand coordinates and the reference (ground truth) coordinates.

Let $\mathbf{x}_i \in \mathbb{R}^3$ denote the predicted coordinate of the $i$-th ligand atom and $\mathbf{x}_i^{\ast} \in \mathbb{R}^3$ denote the corresponding reference coordinate. 
Given $N$ ligand atoms, RMSD is defined as:
\begin{equation}
\mathrm{RMSD} =
\sqrt{
\frac{1}{N}
\sum_{i=1}^{N}
\left\|
\mathbf{x}_i - \mathbf{x}_i^{\ast}
\right\|_2^2
}.
\end{equation}
Lower RMSD values indicate higher agreement between the predicted ligand pose and the reference structure. 

\noindent \textbf{AutoDock Vina-based Binding Metrics.}
To evaluate the predicted protein-ligand binding quality of generated molecular glues, we report three AutoDock Vina-based scores: Vina score, Vina min, and Vina dock. 
Let $P$ denote the protein binding environment, $M$ denote a generated molecule, and $\hat{\mathbf{x}}$ denote its generated 3D conformation. 
We write $E_{\mathrm{Vina}}(P,M,\mathbf{x})$ as the AutoDock Vina scoring function, which approximates the protein-ligand binding free energy in kcal/mol. 
More negative values indicate stronger predicted binding affinity.

The \textbf{Vina score} is computed on the generated ligand pose without further pose optimization:
\begin{equation}
\mathrm{VinaScore}(M)
=
E_{\mathrm{Vina}}(P,M,\hat{\mathbf{x}}).
\end{equation}
This metric evaluates whether the model directly generates a chemically plausible and energetically favourable binding pose.

\textbf{Vina score} further performs local energy minimization around the generated pose before scoring:
\begin{equation}
\mathrm{VinaMin}(M)
=
\min_{\mathbf{x}\in \mathcal{N}(\hat{\mathbf{x}})}
E_{\mathrm{Vina}}(P,M,\mathbf{x}),
\end{equation}
where $\mathcal{N}(\hat{\mathbf{x}})$ denotes the local conformational neighbourhood explored during minimization. 
Compared with the Vina score, Vina min is less sensitive to small local geometric errors in the generated conformation.

\textbf{Vina dock} evaluates the molecule after a full re-docking procedure:
\begin{equation}
\mathrm{VinaDock}(M)
=
\min_{\mathbf{x}\in \Omega(P,M)}
E_{\mathrm{Vina}}(P,M,\mathbf{x}),
\end{equation}
where $\Omega(P,M)$ denotes the docking search space of possible ligand poses in the binding region. 
This metric estimates the best binding score achievable by the generated molecule under the docking protocol, rather than only evaluating the initially generated pose. 
For all three Vina-based metrics, lower values are better.

\noindent \textbf{Quantitative Estimate of Drug-likeness (QED).}
Drug-likeness is evaluated using the Quantitative Estimate of Drug-likeness (QED), which summarizes how similar a molecule is to known drug-like compounds. 
QED combines multiple molecular descriptors, including molecular weight, lipophilicity, topological polar surface area, hydrogen bond donors and acceptors, aromatic rings, rotatable bonds, and structural alerts. 
Given descriptor-specific desirability functions $d_k(M)\in[0,1]$ and weights $w_k$, QED is computed as a weighted geometric mean:
\begin{equation}
\mathrm{QED}(M)
=
\exp\left(
\frac{
\sum_{k=1}^{K} w_k \log d_k(M)
}{
\sum_{k=1}^{K} w_k
}
\right).
\end{equation}
QED ranges from 0 to 1, where larger values indicate more favourable drug-like properties.

\noindent \textbf{Synthetic Accessibility (SA).}
Synthetic accessibility measures how easy a generated molecule is expected to be synthesized. 
The commonly used raw synthetic accessibility score is based on fragment contributions and molecular complexity penalties:
\begin{equation}
\mathrm{SAscore}_{\mathrm{raw}}(M)
=
S_{\mathrm{frag}}(M)
+
S_{\mathrm{complexity}}(M)
+
S_{\mathrm{correction}}(M).
\end{equation}
Here, $S_{\mathrm{frag}}(M)$ reflects whether the molecule contains fragments frequently observed in known synthesized compounds, while $S_{\mathrm{complexity}}(M)$ penalizes structural features that make synthesis difficult, such as large rings, complex ring systems, stereochemical complexity, and large molecular size.

Since the raw SA score is usually defined such that lower values indicate easier synthesis, while our table reports SA with a larger-is-better direction, we use a normalized synthesizability score:
\begin{equation}
\mathrm{SA}(M)
=
\mathrm{clip}
\left(
\frac{10-\mathrm{SAscore}_{\mathrm{raw}}(M)}{9},
0,
1
\right).
\end{equation}
Under this convention, larger SA values indicate molecules that are predicted to be easier to synthesize.

\noindent \textbf{High Affinity Rate (Aff).}
Following common evaluation protocols in target-aware molecule generation, Aff measures the proportion of generated molecules whose predicted binding affinity is stronger than that of the reference ligand for the same target. 
Let $s_{t,j}^{\mathrm{dock}}$ denote the Vina dock score of the $j$-th generated molecule for target $t$, and let $s_t^{\mathrm{ref}}$ denote the corresponding score of the reference ligand. 
Since lower Vina scores indicate stronger predicted binding, Aff is defined as:
\begin{equation}
\mathrm{Aff}
=
\frac{1}{|\mathcal{T}|}
\sum_{t\in\mathcal{T}}
\frac{1}{K_t}
\sum_{j=1}^{K_t}
\mathbb{I}
\left[
s_{t,j}^{\mathrm{dock}} < s_t^{\mathrm{ref}}
\right],
\end{equation}
where $\mathcal{T}$ is the set of test targets, $K_t$ is the number of generated molecules for target $t$, and $\mathbb{I}[\cdot]$ is the indicator function. 
Higher Aff values indicate that a larger fraction of generated molecules are predicted to bind more strongly than the reference ligand. 

\noindent \textbf{DockQ.}
For protein–protein docking, model quality is evaluated using DockQ, a continuous score in the range $[0,1]$ that summarizes interface accuracy by combining three complementary aspects. 
(i) The fraction of native contacts ($F_{\text{nat}}$) measures how many residue–residue contacts observed in the reference complex are correctly reproduced in the prediction, reflecting interface contact recovery. 
(ii) The interface RMSD (iRMSD) quantifies the structural deviation of backbone atoms for residues located at the binding interface after superposition of the complexes, capturing local interface accuracy. 
(iii) The ligand RMSD (LRMSD) measures the backbone deviation of the ligand protein after aligning the receptor, reflecting the global placement of the binding partner. 
These three components are smoothly combined into a single score:
\begin{equation}
\mathrm{DockQ} = \frac{1}{3}
\left(
F_{\text{nat}} +
\frac{1}{1+\left(\frac{\mathrm{iRMSD}}{1.5}\right)^2} +
\frac{1}{1+\left(\frac{\mathrm{LRMSD}}{8.5}\right)^2}
\right).
\end{equation}

\subsection{Ellipsoid Interface Prediction Metric Details}\label{app:ellipsoid_metric}

Each interface is represented as a Gaussian ellipsoid 
$\mathcal{N}(\mu,\Sigma)$, where $\mu\in\mathbb{R}^3$ denotes the ellipsoid center and 
$\Sigma\in\mathbb{R}^{3\times 3}$ is a symmetric positive definite (SPD) matrix describing its orientation and spatial extent. 
Given a prediction $(\mu_p,\Sigma_p)$ and reference $(\mu_r,\Sigma_r)$, we evaluate complementary geometric and distributional discrepancies using the following metrics.

\textbf{Center-L2}.  
This metric measures the Euclidean distance between predicted and reference ellipsoid centers:
\begin{equation}
d_{\text{center}}=\|\mu_p-\mu_r\|_2.
\end{equation}

\textbf{LogDet Error}.  
Differences in ellipsoid scale are quantified by comparing the log-determinants of the covariance matrices:
\begin{equation}
d_{\log\det}= \left|\log\det(\Sigma_p)-\log\det(\Sigma_r)\right|.
\end{equation}

\textbf{Wasserstein-2 distance (W2 Dist)}.  
This geometry-aware distance jointly evaluates discrepancies in both the ellipsoid center and shape:
\begin{equation}
W_2^2
=\|\mu_p-\mu_r\|_2^2
+\mathrm{Tr}\!\left(
\Sigma_p+\Sigma_r
-2\left(\Sigma_r^{1/2}\Sigma_p\Sigma_r^{1/2}\right)^{1/2}
\right),
\end{equation}
and we report $W_2=\sqrt{W_2^2}$.

\textbf{Jensen--Shannon divergence (JSD)}. 
Distributional similarity between the predicted and reference Gaussian densities is assessed using JSD:
\begin{equation}
\mathrm{JSD}(P\|Q)
=\frac{1}{2}\mathrm{KL}(P\|M)+\frac{1}{2}\mathrm{KL}(Q\|M),
\quad
M=\frac{1}{2}(P+Q),
\end{equation}
where $P=\mathcal{N}(\mu_r,\Sigma_r)$ and $Q=\mathcal{N}(\mu_p,\Sigma_p)$. 
Because $M$ is a Gaussian mixture rather than a single Gaussian, JSD is estimated via Monte Carlo:
\begin{equation}
\widehat{\mathrm{JSD}}
=\frac{1}{2}\left[
\frac{1}{N}\sum_{x_i\sim P}\log\frac{p(x_i)}{m(x_i)}
+
\frac{1}{N}\sum_{y_j\sim Q}\log\frac{q(y_j)}{m(y_j)}
\right],
\end{equation}
where $m(\cdot)=\frac{1}{2}\big(p(\cdot)+q(\cdot)\big)$.

\section{Group Analysis}
We evaluate the structural quality of the generated ternary complexes from a structure prediction perspective. 
For each group, molecular glues sampled by {\ours} on the test set are combined with the sequences of the two proteins and fed into AlphaFold3~\citep{alphafold3} for complex structure prediction, after which pTM and ipTM confidence are computed. 
As shown in Fig.~\ref{fig:alphafold_group}, AlphaFold3 predicts structurally reasonable ternary complexes for most groups, providing additional evidence for the quality of the sampled molecular glues. 
An exception is observed for the \texttt{6TZ8} family, which corresponds to the immunosuppressant drug \texttt{FK506}~\citep{fk506}, while TriGlue also achieves a relatively low DockQ score of 0.13 for this group. 

\begin{figure}[H]
    \centering
    \includegraphics[width=0.70\linewidth]{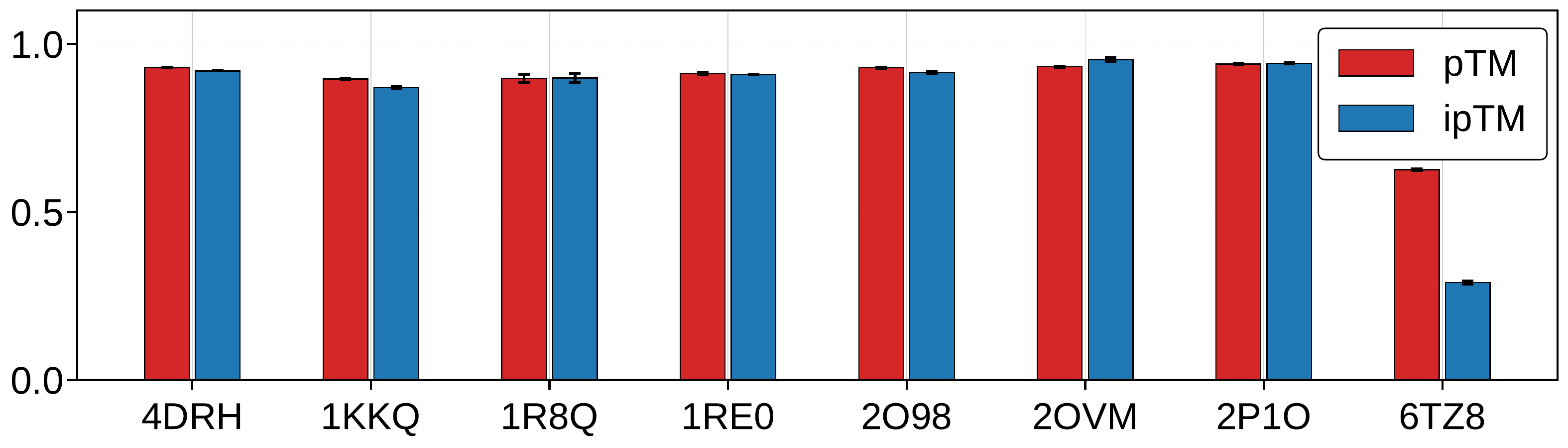}
    \caption{pTM and ipTM confidence by group.}
    \label{fig:alphafold_group}
\end{figure}



\section{Broader Impacts}\label{app:broader_impacts}
This work introduces a generative framework for molecular glue design, a modality with the potential to expand the druggable proteome and accelerate the discovery of therapies for diseases that remain difficult to treat using conventional small-molecule approaches. 
By enabling the computational exploration of ternary complex formation, our approach may reduce experimental trial-and-error, shorten early-stage drug discovery timelines, and lower research costs, which could ultimately improve access to novel therapeutics and benefit public health. 
Furthermore, the proposed biology-inspired generative modeling paradigm may stimulate interdisciplinary collaboration across machine learning, structural biology, and medicinal chemistry.

\newpage

\end{document}